\documentclass[final]{l4dc2026}

\usepackage{graphicx}

\DeclareMathOperator*{\argmin}{arg\,min}
\title[Linear Denoisers]{Precise Performance of Linear Denoisers in the Proportional Regime}
\usepackage{times}
\usepackage{bm}

\coltauthor{\Name{Reza Ghane$^*$} \Email{rghanekh@caltech.edu}\\
 \Name{Danil Akhtiamov$^*$} \Email{dakhtiam@caltech.edu}\\
 \Name{Babak Hassibi} \Email{hassibi@caltech.edu}\\
 \addr California Institute of Technology, Pasadena, CA, 91125}

\newcommand{\bSigma}{\bm{\Sigma}}

\newcommand{\blambda}{\bm{\lambda}}
\newcommand{\btheta}{\bm{\theta}}

\newcommand{\bA}{\mathbf{A}}

\newcommand{\bG}{\mathbf{G}}

\newcommand{\bI}{\mathbf{I}}

\newcommand{\bR}{\mathbf{R}}
\newcommand{\bS}{\mathbf{S}}

\newcommand{\bW}{\mathbf{W}}
\newcommand{\bX}{\mathbf{X}}

\newcommand{\bZ}{\mathbf{Z}}

\newcommand{\bb}{\mathbf{b}}

\newcommand{\be}{\mathbf{e}}

\newcommand{\bg}{\mathbf{g}}
\newcommand{\bh}{\mathbf{h}}

\newcommand{\bu}{\mathbf{u}}
\newcommand{\bv}{\mathbf{v}}
\newcommand{\bw}{\mathbf{w}}
\newcommand{\bx}{\mathbf{x}}

\newcommand{\bz}{\mathbf{z}}

\newcommand{\hbX}{\hat{\mathbf{X}}}

\newcommand{\bbE}{\mathbb{E}}
\newcommand{\bbR}{\mathbb{R}}

\newcommand{\bbP}{\mathbb{P}}

\newcommand{\calE}{\mathcal{E}}

\newcommand{\calL}{\mathcal{L}}
\newcommand{\calN}{\mathcal{N}}

\newcommand{\calS}{\mathcal{S}}

\newcommand{\bzero}{\mathbf{0}}
\newcommand{\bone}{\mathbf{1}}

\newcommand{\tr}{\text{Tr}}

\newcommand{\rarrowp}{\xrightarrow[]{\bbP}}

\newtheorem{assumption}{Assumptions}
\begin{document}
\maketitle

\makeatletter
\renewcommand\@makefntext[1]{#1}
\makeatother
\renewcommand{\thefootnote}{}
\footnotetext{$^*$Equal Contribution}
\makeatletter
\renewcommand\@makefntext[1]{\@textsuperscript{\normalfont\@thefnmark}~#1}
\makeatother
\renewcommand{\thefootnote}{\arabic{footnote}}

\begin{abstract}%

In the present paper we study the performance of linear denoisers for noisy data of the form $\bx + \bz$, where $\bx \in \bbR^d$ is the desired data with zero mean and unknown covariance $\bSigma$, and $\bz \sim \mathcal{N}(0, \bSigma_\bz)$ is additive noise. Since the covariance $\bSigma$ is not known, the standard Wiener filter cannot be employed for denoising. Instead we assume we are given samples $\bx_1,\dots,\bx_n \in \bbR^d$ from the true distribution. A standard approach would then be to estimate $\bSigma$ from the samples and use it to construct an ``empirical" Wiener filter. However, in this paper, motivated by the denoising step in diffusion models, we take a different approach whereby we train a linear denoiser $\bW$ from the data itself. In particular, we synthetically construct noisy samples $\hat{\bx}_i$ of the data by injecting the samples with Gaussian noise with covariance $\bSigma_1 \neq \bSigma_{\bz}$ and find the best $\bW$ that approximates $\bW\hat{\bx}_i \approx \bx_i$ in a least-squares sense. In the proportional regime $\frac{n}{d} \rightarrow \kappa > 1$ we use the {\it Convex Gaussian Min-Max Theorem (CGMT)} to analytically find the closed form expression for the generalization error of the denoiser obtained from this process. Using this expression one can optimize over $\bSigma_1$ to find the best possible denoiser. Our numerical simulations show that our denoiser outperforms the ``empirical" Wiener filter in many scenarios and approaches the optimal Wiener filter as $\kappa\rightarrow\infty$. 

\end{abstract}


\section{Introduction and Related Works}

Data denoising is a fundamental problem arising in many areas of data science, signal processing, control theory, and machine learning. Denoisers have been successfully applied to signal processing \citep{kailath2001linear}, image processing \citep{hirakawa2006image, zhang2017beyond, zamir2022restormer} and generative AI \citep{ songscore, song2019generative} .

Arguably the most widely known denoiser, the Wiener Filter \citep{wiener1964extrapolation, kolmogorov1941stationary}, dates back to the 1940s. However, constructing such filter in practice requires precise knowledge of the covariance of the data $\bx \in \bbR^d$. Noting that reliable estimation of the second-order statistics requires access to plethora of samples, other works, such as \cite{hirakawa2006image}, attempt to circumvent such need by {\it learning} denoisers {\it directly from data}. 

To provide a pertinent and opportune motivation, denoisers are used as central building blocks in modern diffusion models \citep{songscore, song2019generative}. In the context of diffusion, the denoiser $D_{\btheta}: \bbR^d \rightarrow \bbR^d$ is a neural network parametrized by its weights $\btheta \in \bbR^D$, such as UNet \citep{ho2020denoising} or  Diffusion Transformer (DiT) \citep{peebles2023scalable}. Simplifying and omitting certain technical details for ease of explanation, one can illustrate the core idea of  denoising diffusion models as follows. Suppose we trained the denoiser $D_{\btheta}$ to learn a ``good quality" approximation $D_{\btheta}(\bx_i + \bz_i) \approx \bx_i$
where $\bz_i \sim \mathcal{N}(0, \sigma_{\bz}^2\bI_d)$ is artificially injected noise such that the signal to noise ratio (SNR) is very low, i.e. 
\begin{align}\label{eq: diff_snr}
d\sigma_{\bz}^2 \approx \|\bz_i\|^2 \gg \|\bx_i\|_2^2 
\end{align}
Then, under \eqref{eq: diff_snr}, it would be reasonable to attempt generating new samples by dropping dependence of the denoiser on $\bx$  as $\bx \sim D_{\btheta}(\bz) \text{ where } \bz \sim \mathcal{N}(0, \sigma_{\bz}^2\bI_d)$.

Motivated by classical works on learning denoisers as well as more recent works on diffusion, we would like to initiate a systematic study of the performances of  denoisers. 
As a starting point, we propose to study linear denoisers trained via the least squares objective. 
If the number of samples $n$ and the dimensionality of data $d$ satisfy $n \gg d$, it is known that the least squares solution will converge to the Wiener Filter, whose performance has been studied extensively in the literature. As such, to make the problem interesting, we assume that $n$ is proportional to $d$.

In the proportional regime $n$, the number of samples, and $d$, the number of features, grow together such that $\frac{n}{d} \rightarrow \kappa$. This is different from classical works in statistics and control theory, where, usually, $d$ is fixed and $n \rightarrow \infty$. Arguably, the main technical difference between the two is that one can reliably estimate the covariance of the data reliably when $n \gg d^2$, but, generally speaking, cannot do so if $n = \theta(d)$. The latter in particular means that the Wiener Filter cannot be recovered precisely in the proportional regime, as it is defined based on the covariance of the data. 

Understanding the precise asymptotics of various statistical inference problems in the proportional regime has been the subject of study of many works in the past decade \citep{thrampoulidis2018precise, mallory2025universality,hastie2022surprises, li2025physics, wang2025glamp, huang2025universality, loureiro2021learning, deng2022model, akhtiamov2024regularized, ghane2025one, akhtiamov2025precise, ghane2024universality}. Analyzing the performances of learning algorithms in the proportional regime requires a  different toolbox from classical statistical work, the main two of which have proven to be the Approximate Message Passing (AMP) \citep{donoho2009message} and the Convex Gaussian Min-Max (CGMT) \citep{thrampoulidis2014gaussian} frameworks. For the purposes of the present work, we employ the latter (CGMT) approach. 

\section{Problem Setting and Preliminaries}

We use bold font for vectors and matrices and normal font for scalars in this exposition. Consider the following data denoising problem: 

\bigskip

\fbox{\parbox{\textwidth}{Given $n$ i.i.d data points sampled from the normal distribution $\bx \sim \mathcal{N}(0, \bSigma)$ with {\bf an unknown} covariance $\bSigma \in \bbR^{d\times d}$ and a noise distribution $\bz \sim \mathcal{N}(0,\bSigma_\bz)$ with a {\bf known} covariance $\bSigma_\bz  \in \bbR^{d\times d} $, design a linear denoiser $\bW \in \bbR^{d \times d}$, that takes noisy data of the form $\hat{\bx} = \bx + \bz$ as input and ``suppresses the noise" and outputs $\bW\hat{\bx} \approx \bx$ as accurately as possible. More formally, $\bW$ should minimize the following generalization error objective: \begin{align}\label{eq: gen_err}
 \calE(\bW):=\bbE_{\bz \sim \mathcal{N}(0, \bSigma_\bz)}  \bbE_      {\bx \sim \mathcal{N}(0,\bSigma)}\|\bW^T (\bx + \bz) - \bx\|_2^2 \end{align}}}
 
\bigskip

Note that \eqref{eq: gen_err} can be simplified as:
\begin{align}\label{eq: gen_err_tr}
    \calE(\bW)=\bbE_\bz\bbE_\bx \|(\bW^T - \bI) \bx + \bW^T \bz\|_2^2 = \tr \left((\bW^T - \bI) \bSigma (\bW^T - \bI)^T\right) +  \tr\left( \bW^T \bSigma_\bz \bW\right)
\end{align}
Minimizing \eqref{eq: gen_err_tr} over $\bW$ directly leads to 
\begin{align}\label{eq: wiener}
    \bW = \left(\bSigma + \bSigma_\bz\right)^{-1}\bSigma
\end{align}
The denoiser \eqref{eq: wiener} is known as the {\it Wiener Filter} in the literature. By construction, \eqref{eq: wiener} achieves the optimal generalization error among all linear denoisers. However, finding \eqref{eq: wiener} requires precise knowledge of $\bSigma$. Recovering $\bSigma$ from data with good accuracy might not be feasible when $n$ and $d$ are proportional to each other, unlike the classical case $n \gg d$. As such, one natural approach to the denoising problem would be to take the {\it empirical Wiener Filter}
\begin{align}\label{eq: emp_wiener}
    \bW = \left(\hat{\bSigma} + \bSigma_\bz\right)^{-1}\hat{\bSigma}
\end{align}
Where the empirical covariance $\hat{\bSigma}$ is defined as follows via the sample data matrix 
\\ $\bX^T = \begin{pmatrix}
    \bX^1 & \bX^2 & \hdots \bX^n
\end{pmatrix}\allowbreak \in \bbR^{d \times n}$ with $\bX^i \in \bbR^d$ being the rows generated independently:
\begin{align}\label{eq: emp_cov}
    \hat{\bSigma} = \frac{1}{n-1}\bX^T\left(\bI_n - \frac{\bone_n\bone_n^T}{n}\right)\bX
\end{align}
Naturally, the quality of the denoiser \eqref{eq: emp_wiener} depends on the quality of approximation of the true covariance $\bSigma$ by the empirical covariance \eqref{eq: emp_cov}. The latter does not have to be good in general, as $\bSigma \in \bbR^{d \times d}$ has $\Theta(d^2)$ degrees of freedom and \eqref{eq: emp_cov} attempts to estimate $\bSigma$ from $n = \Theta(d)$ samples. As such, we would like to consider an alternative approach to minimizing \eqref{eq: gen_err} as well. The approach we propose is outlined below and is based on learning $\bW$ directly from data, but the main technicality lies in how the training data should be generated from $\bX$.  

We divide the training data $\bX$ into $N$ batches $\bX_1 \in \bbR^{n_1 \times d}, \dots, \bX_N \in \bbR^{n_N \times d}$
 In other words, $\bX_1 \in \bbR^{n_1 \times d}, \dots, \bX_N \in \bbR^{n_N \times d}$ are defined via $\bX^T =  \begin{pmatrix}
        \bX_1^T &
        \bX_2^T &
        \hdots &
        \bX_N^T
    \end{pmatrix} $
. Using $\bX_t^i$ to denote the $i$th row of $\bX_t$, note that the data points $\bX^{i}_t$ and $\bX^{i'}_{t'}$ are independent whenever $(i,t) \ne (i',t')$ and each row satisfies $$\bX^{i}_t \sim \calN(\bzero, \bSigma) \text{ for } t = 1,\dots, N\text{ and }i=1,\dots,n_t$$ 

Furthermore, consider independent noise vectors
$$\bZ^{i}_t \in \bbR^{d} \text{ where } \bZ^{i}_t \sim \calN(\bzero, \bSigma_{t}) ,\text{ } t=1,\dots,N, i=1,\dots,n_t$$

We construct noisy data batches $\hbX_t := \bX_t + \bZ_t,\text{ } t=1,\dots,N$ and train a denoiser $\bW_{\text{lsq}}$ as: 
\begin{align}\label{eq: denoiser_ls}
    \bW_{\text{lsq}} := \argmin_\bW \Bigl\|
        \hbX \bW - 
        \bX
    \Bigr\|_F^2
\end{align}

Here, $\hbX$ is the full noisy data matrix defined as
$\hbX^T = \begin{pmatrix}
        \hbX_1^T &
        \hbX_2^T &
        \hdots &
        \hbX_N^T
    \end{pmatrix} 
$. The main results of the present work, namely Theorems \ref{thm: N=1} and \ref{thm: arb_N} in Section \ref{sec: main}, address the following questions:

\bigskip

\fbox{\parbox{\textwidth}{What is the  generalization error \eqref{eq: gen_err} of the denoiser $\bW_{\text{lsq}}$ found via \eqref{eq: denoiser_ls} for given $N, n_1,\dots,n_t,d, \bSigma_1, \dots, \bSigma_N$?  }}

\bigskip

Focusing on the case of $N=1$ and relying on Theorem \ref{thm: N=1}, we also conduct extensive mathematical (Corollary \ref{cor: scal}) as well as numerical (Section \ref{sec: experiments}) investigations of the following questions:

\bigskip
\fbox{\parbox{\textwidth}{What is the optimal  $\bSigma_1$ leading to the minimal generalization error \eqref{eq: gen_err} among linear denoisers trained via \eqref{eq: denoiser_ls} in terms of $\bSigma$ and $\bSigma_\bz$? How can we approximate this optimal $\bSigma_1$ algorithmically given access only to $\bx_1,\dots,\bx_n$? Can we outperform the empirical Wiener filter \eqref{eq: emp_wiener} using such an approximation of the optimal $\bSigma_1$?}}

\bigskip


We will also impose the following technical assumptions that are necessary for our analyses.
\begin{assumption}\label{ass: main}
    \begin{enumerate}
        \item $\bSigma$ is invertible.
        \item $\frac{n_t}{d} \to  \kappa_t > 0$ as $n_t, d \rightarrow \infty$ for each $t=1,\dots, N$ and $\kappa = \kappa_1 + \dots + \kappa_N > 1$. 
    \end{enumerate}
\end{assumption}
\textbf{Our Main Contributions:}
\begin{itemize}
    \item We characterize the generalization error of $\bW_{\text{lsq}}$ found from \eqref{eq: denoiser_ls} for arbitrary $N, n_1,\dots,n_t,d, \allowbreak \bSigma_1, \dots, \bSigma_N$ satisfying Assumptions \ref{ass: main} in Theorem \ref{thm: arb_N}.
    \item In the case of $N=1$, we derive much more explicit results, which are formulated  in Theorem \ref{thm: N=1}, again assuming assumptions \ref{ass: main} hold. 
    \item We use Theorem \ref{thm: N=1} to find the optimal $\bSigma_1$ minimizing the generalization error for the speciailized case when $N=1$ and both $\bSigma$ and $\bSigma_\bz$ are scalar. Perhaps surprisingly, even in this case we see that the optimal $\bSigma_1$ is in general different from $\bSigma_\bz$, meaning it's better to take different noise distributions for training and testing. The result is presented in Corollary \ref{cor: scal}. 
    \item In Section \ref{sec: experiments}, we propose an algorithm for approximating the optimal $\bSigma_1$ for the case $N=1$ and arbitrary $\bSigma, \bSigma_\bz$. This algorithm is based on Theorem \ref{thm: N=1}.  We then proceed to verify that in certain cases the $\bW_{\text{lsq}}$ resulting from \eqref{eq: denoiser_ls} indeed outperforms the empirical Wiener filter \eqref{eq: emp_wiener}.  
\end{itemize}

\section{Main Results}
\label{sec: main}
 Since we use CGMT, our results are asymptotic, meaning that we characterize the error $\calE(\bW_{\text{lsq}})$ via certain quantities $f_{n,d}$, such that $\frac{\calE(\bW_{\text{lsq}})}{f_{n,d}} \to 1$ as $n,d \to \infty$ and $\frac{n}{d} \to \kappa > 1$. For the ease of exposition, we just write equalities within the results of this section, but keep the reference to the asymptotic regime in the formulation to avoid confusion. We would like to begin with stating the results we obtained for $N=1$, as the expressions we get in this case are more explicit and thus allow for more insight.
\begin{theorem}\label{thm: N=1}
    Assume that $N=1$ and Assumptions \ref{ass: main} hold. Then, for the denoiser $\bW_{\text{lsq}}$ found via \eqref{eq: denoiser_ls}, its generalization error \eqref{eq: gen_err} can be characterized asymptotically as follows: 
    \begin{align} \label{eq: gen_thm1}
       \calE(\bW_{\text{lsq}}) = \Bigl(1 + \frac{1}{n-d} \tr \Bigl[ \Bigl(\bSigma + \bSigma_1 \Bigr)^{-1} \Bigl(\bSigma + \bSigma_{\bz}\Bigr)  \Bigr] \Bigr) \Bigl( \tr \Bigl[\bSigma\Bigr] - \tr \Bigl[\Bigl(\bSigma + \bSigma_1\Bigr)^{-1} \bSigma^2 \Bigr)\Bigr]\Bigr) \nonumber  \\
    + \tr \Bigl[ \Bigl(\bSigma + \bSigma_1 \Bigr)^{-1} \Bigl(\bSigma_{\bz} - \bSigma_1\Bigr) \Bigl(\bSigma + \bSigma_1 \Bigr)^{-1} \bSigma^2\Bigr]
    \end{align}
\end{theorem}
\begin{figure}[htb]
    \centering
    \includegraphics[width=0.5\linewidth]{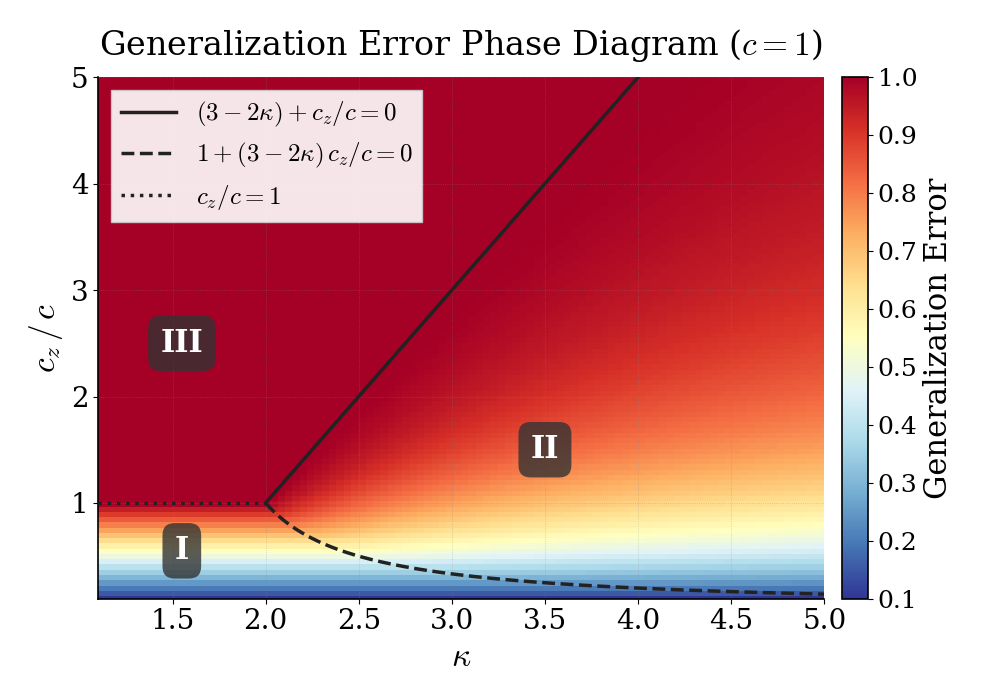}
    \caption{Phase Transition explained in Remark \ref{rem: pht}}
    \label{fig:phase}
\end{figure}
Specifying Theorem \ref{thm: N=1} to the case $\bSigma_1 = \bSigma_\bz$ leads to the following corollary:
\begin{corollary}\label{cor: lsq_wiener}
      Under the setting of Theorem \ref{thm: N=1}, assume in addition that  $\bSigma_1 = \bSigma_{\bz}$ and $\bSigma_{\bz}$ and $\bSigma$ commute. Then, for the denoiser found via \eqref{eq: denoiser_ls}, its generalization error \eqref{eq: gen_err} asymptotically equals the following, where $\cal{E}_{\text{Wiener}}$ denotes the generalization error \eqref{eq: gen_err} minimized over $\bW$, which is attained by the Wiener Filter \eqref{eq: wiener}: 
    \begin{align*}
         \calE(\bW_{\text{lsq}})= \frac{\kappa}{\kappa-1} \cal{E}_{\text{Wiener}}
    \end{align*}
    Note that taking $\kappa \to \infty$, i.e. assuming $n \gg d$,  recovers the performance of the Wiener filter. As expected, $\calE(\bW_{\text{lsq}}) \ge  \cal{E}_{\text{Wiener}}$.
\end{corollary}
We obtain another corollary of Theorem \ref{thm: N=1} by treating the case of scalar $\bSigma$ and $\bSigma_\bz$ in greater detail:
\begin{corollary}\label{cor: scal}
    Under the setting of Theorem \ref{thm: N=1}, assume in addition that $\bSigma$ and $\bSigma_\bz$ are scalar matrices and define:
    $$\bSigma = \frac{c}{d}\bI_d \text{ and } \bSigma_\bz = \frac{c_\bz}{d}\bI_d $$
    Then minimizing the generalization error \eqref{eq: gen_err} of the denoiser found via \eqref{eq: denoiser_ls} over $\bSigma_1$ leads to the following optimal generalization error:  
    \begin{align} \label{eq: opt_scal}
    \begin{cases}
        \frac{-(c+ c_{\bz})^2 +  4(\kappa-1)^2 c c_{\bz}}{4 (\kappa-2) (\kappa-1) ( c + c_{\bz}) }  & (3-2\kappa) c + c_{\bz} < 0 , \quad c + (3 - 2\kappa) c_{\bz} < 0 \\
        \min\{c, c_{\bz}\} & \text{otherwise}
    \end{cases} 
\end{align}
\end{corollary}
\begin{remark}\label{rem: pht}
Corollary \ref{cor: scal} reveals an interesting \textit{phase transition} phenomenon, whose \textit{phase diagram} is depicted in Fig. \ref{fig:phase}. In region I, defined by $3-2\kappa + \frac{c_\bz}{c} > 0$, the generalization error equals  $c_{\bz}$ and the optimal $\bW_{\text{lsq}} = \bI_d$ regardless of the choice of $\bSigma_1$. In region III, defined by $3-2\kappa + \frac{c}{c_\bz} > 0$, the generalization error equals $c$ and the optimal $\bW_{\text{lsq}} = 0$, again regardless of the choice of $\bSigma_1$. The interesting region II,  defined as the complement of I and III, is where there exists a unique  $\bSigma_1$ achieving the best performance and the resulting $\bW_{\text{lsq}} \notin \{0, \bI_d\}$. We fixed $c=1$ for Fig. \ref{fig:phase}.
\end{remark}
Finally, we formulate the expression that we have derived for the case of $N>1$. The expression we obtain for $N>1$ is more complicated and we leave analyzing it further and making it more explicit as a subject for future work.
\begin{theorem}\label{thm: arb_N}
     Suppose that Assumptions \ref{ass: main} hold. Consider $\{\theta_t\}_{t=1}^N$ found as the unique fixed points of the following system of equations
    \begin{align*}
    \theta_t &= 2 - 2 \theta_t \tr \biggl[\Bigl(\sum_{t=1}^N \theta_t n_t (\bSigma_t + \bSigma) \Bigr)^{-1} \bSigma_t\biggr]
    \end{align*}
    Define $$\theta  = 2 -  4  \tr\biggl[\Bigl(\sum_{t=1}^N \theta_t n_t (\bSigma_t + \bSigma) \Bigr)^{-1} \bSigma\biggr]$$ 
    Finally, consider the matrix $\bA \in \bbR^{(N+2) \times (N+2)}$ defined in \eqref{eq: A_def} and vector $\bb \in \bbR^{N+2}$ defined in \eqref{eq:b_def}, deferred to the Appendix due to their longer descriptions. 
    Then, for the denoiser $\bW_{\text{lsq}}$ found via \eqref{eq: denoiser_ls}, its generalization error \eqref{eq: gen_err} can be characterized asymptotically as follows: 
    \begin{align*}
     \calE(\bW_{\text{lsq}}) = \frac{1}{4(n-d)^2} (\bA^{-1} \bb)_{N+2} +  \tr \biggl[ \Bigl(\sum_{t=1}^N \theta_t n_t (\bSigma_t + \bSigma) \Bigr)^{-1} \bSigma_{\bz}  \Bigl(\sum_{t=1}^N \theta_t n_t (\bSigma_t + \bSigma) \Bigr)^{-1} \\
    \cdot \Bigl( 4\theta^2 (\bA^{-1} \bb)_{1} \bSigma + 16(n-d)^2 \bSigma^2 + \theta^2 \sum_{t=1}^N \theta^2_t \bSigma_t (\bA^{-1} \bb)_{t+1}   \Bigr) \biggr]
    \end{align*}
    Where $(\bA^{-1} \bb)_j$ denotes the $j$th entry of the vector $\bA^{-1} \bb \in \bbR^{N+2}$.
\end{theorem}

\section{Proof Technique and Sketches}

Both proofs rely heavily on a result known in the literature as the Convex Gaussian Min-Max Theorem (CGMT). We refer the reader to \cite{thrampoulidis2014gaussian, akhtiamov2025novel} for a formal detailed treatment of the CGMT, but we will outline it briefly for completeness. The goal of the CGMT is to analyze properties of solutions of the objectives of the following form, called the {\it Primary Optimization (PO)} objective:
\begin{align}\label{eq: PO}
    \min_{\bw \in \calS_\bw}  \max_{\bu \in \calS_\bu} \bu^T\bG\bw + \psi(\bu,\bw)\quad \textbf{(PO)}
\end{align}
Here, $\bu \in \bbR^n$, $\bw \in \bbR^d$, $\calS_\bu \subset \bbR^n, \calS_\bw \subset \bbR^d$ are compact convex sets, $\psi$ is convex in $\bu$ and concave in $\bw$ and $\bG$ is an i.i.d. $\mathcal{N}(0,1)$ {\it matrix}.  To analyze \eqref{eq: PO}, the CGMT framework introduces another objective, called {\it Auxillary Objective (AO)}, where $\bg \in \bbR^d, \bh \in \bbR^n$ are i.i.d. $\mathcal{N}(0,1)$ {\it vectors}:  
\begin{align}\label{eq: AO}
    \min_{\bw \in \calS_\bw}  \max_{\bu \in \calS_\bu} \bg^T\bw\|\bu\| + \bh^T\bu\|\bw\| +  \psi(\bu,\bw)\quad \textbf{(AO)}
\end{align}
Consider an arbitrary set $\bS \subset \bS_\bw$. Let $\bw_{AO}$ and $\bw_{PO}$ be any minimizers of \eqref{eq: AO} and \eqref{eq: PO} respectively. Assume that $\bw_{AO} \in \bS$ holds with probability approaching 1 (w.p.a. 1), i.e. :
$$\mathbb{P}\left(\bw_{AO} \in \bS \right) \to 1 \text{ as }n, d \to \infty$$
Then $\bw_{PO} \in \bS$ holds w.p.a. 1 as well, i.e.:
$$\mathbb{P}\left(\bw_{PO} \in \bS \right) \to 1 \text{ as }n, d \to \infty$$
Note that, formally speaking, we should introduce sequences of PO and AO to discuss convergence in probability. However, we suppress these formalities to ease exposition and refer the interested reader to \cite{thrampoulidis2014gaussian, akhtiamov2025novel} for a rigorous exposition.

\textbf{Sketch of the proof of Theorems \ref{thm: N=1} and \ref{thm: arb_N}}: Recall that the linear denoiser $\bW_{\text{lsq}}$ is defined via:
\begin{align}\label{eq: denoise_recall}
     \min_\bW \Bigl\|
        \hbX \bW - 
        \bX
    \Bigr\|_F^2 = \min_{\bW} \|\bX (\bW - \bI) + \bZ \bW \|_F^2
\end{align}
Denoting the $i$-th column of $\bW$ by $\bw_i$, we can rewrite \eqref{eq: denoise_recall} as:
\begin{align}
    \sum_{i=1}^d  \min_{\bw_i} \|\bX (\bw_i - \be_i) + \bZ \bw_i \|^2
\end{align}
Introducing a set of Fenchel dual variables $\bv_i \in \bbR^d$ for $i=1,\dots,d$ and writing $\bX = \bG \bSigma^{1/2}$, where $\bG \in \bbR^{n\times d}$ is i.i.d. standard normal, we obtain:
\begin{align*}
    \sum_{i=1}^d \min_{\bw_i} \max_{\bv_i}  \bv_i^T \bG \bSigma^{1/2} (\bw_i - \be_i) + \bv_i^T \bZ \bw_i - \frac{\|\bv_i\|_2^2}{4}
\end{align*}
With a slight abuse of notation we drop the indices from $\bw_i$ and $\bv_i$ and rewrite:
\begin{align*}
    \sum_{i=1}^d \min_{\bw} \max_{\bv}  \bv^T \bG \bSigma^{1/2} (\bw - \be_i) + \bv^T \bZ \bw - \frac{\|\bv\|_2^2}{4}
\end{align*}
Let $\bu:= \bSigma^{1/2} (\bw - \be_i)$. Using a Lagrange multiplier and leveraging Sion's minimax theorem:
\begin{align}\label{eq: PO_N=1}
     \sum_{i=1}^d \max_{\blambda} \min_{\bw,\bu} \max_{\bv}  \bv^T \bG \bu + \bv^T \bZ \bw + \blambda^T \bu - \blambda ^T \bSigma^{1/2} (\bw - \be_i)  - \frac{\|\bv\|_2^2}{4}
\end{align}
Treating each term of \eqref{eq: PO_N=1} as a separate PO with an i.i.d. standard normal $\bG$, we arrive at the following sum of AOs:
\begin{align*}
     \sum_{i=1}^d\max_{\blambda} \min_{\bw, \bu} \max_{\bv}  \|\bu\|_2 \bh^T \bv + \|\bv\|_2 \bg^T\bu  + \bv^T \bZ \bw + \blambda^T \bu - \blambda ^T \bSigma^{1/2} (\bw - \be_i)  - \frac{\|\bv\|_2^2}{4}
\end{align*}
Denoting $\beta = \|\bv\|_2$ and performing optimization over the direction $\frac{\bv}{\|\bv\|}$, we have
\begin{align*} \sum_{i=1}^d 
     \max_{\blambda} \min_{\bw,\bu} \max_{\beta>0}  \beta \bg^T\bu + \beta \Bigl\|\bZ \bw +   \|\bu\|_2 \bh \Bigr\|_2 + \blambda^T \bu - \blambda ^T \bSigma^{1/2} (\bw - \be_i)  - \frac{\beta^2}{4}
\end{align*}
 Denoting $\eta = \|\bu\|_2$ and optimizing over $\frac{\bu}{\|\bu\|}$, we obtain
\begin{align*}
     \sum_{i=1}^d  \max_{\blambda} \min_{\bw, \eta > 0} \max_{\beta>0}  -\eta  \|\beta \bg + \blambda\|_2+ \beta \Bigl\|\bZ \bw +   \eta \bh \Bigr\|_2 - \blambda ^T \bSigma^{1/2} (\bw - \be_i)  - \frac{\beta^2}{4}
\end{align*}
After swapping around min and max due to convexity-concavity by Lemma \ref{lem: conv_beta}, we employ the square-root trick $\sqrt{x} = \min_{s>0} \frac{s}{2} + \frac{x}{2s}$ and obtain:
\begin{align*}
    \sum_{i=1}^d  \max_{\blambda} \min_{\eta > 0} \max_{\beta>0}  -\eta  \|\beta \bg + \blambda\|_2 - \frac{\beta^2}{4} + \min_{\bw, \tau>0} \frac{\beta \tau}{2} + \frac{\beta}{2\tau} \Bigl\|\bZ \bw +   \eta \bh \Bigr\|^2_2  - \blambda ^T \bSigma^{1/2} (\bw - \be_i)
\end{align*}
Now, recall the block structure we have for $\bZ$ and introduce a matching structure for $\bh$, where each $\bh^{(t)} \in \bbR^{n_t}$:
    $\bZ^T =  \begin{pmatrix}
        \bZ_1^T &
        \bZ_2^T &
        \hdots &
        \bZ_N^T
    \end{pmatrix} \text{ and } \bh^T =  \begin{pmatrix}
        \bh^{(1) T} &
        \bh^{(2) T}&
        \hdots &
        \bh^{(N) T}\end{pmatrix}$. 
We arrive at:
\begin{align*}
    \sum_{i=1}^d \max_{\blambda} \min_{\eta > 0} \max_{\beta>0}  -\eta  \|\beta \bg + \blambda\|_2 - \frac{\beta^2}{4} + \min_{\bw, \tau>0} \frac{\beta \tau}{2} + \frac{\beta}{2\tau}\sum_{t=1}^N \Bigl\|\bZ_t \bw +   \eta \bh^{(t)} \Bigr\|^2_2  - \blambda ^T \bSigma^{1/2} (\bw - \be_i)
\end{align*}
Using a Fenchel Dual $\bv_t \in \bbR^{n_t}$ for each $t=1,\dots,N$, we obtain:
\begin{align}\label{eq: interm_obj}
    \sum_{i=1}^d  \min_{\bw} \max_{\bv_t}\frac{\beta}{2\tau} \Bigl( \sum_{t=1}^N \bv_t^T \bZ_t \bw +   \eta \bv_t^T\bh^{(t)} - \frac{\|\bv_t\|_2^2}{4} \Bigr)  - \blambda ^T \bSigma^{1/2} (\bw - \be_i)
\end{align}
Applying CGMT again, but this time with respect to the randomness in $\bZ^{(t)}$, we can replace the inner optimization part of \eqref{eq: interm_obj} by the following:
\begin{align}\label{eq: master_ao}
    \min_{\bw} \max_{\bv_t}  \frac{\beta}{2\tau} \Bigl( \sum_{t=1}^N \|\bSigma_t^{1/2}\bw\|_2 \bh_t^T \bv_t + \|\bv_t\|_2 \bg_t^T \bSigma_t^{1/2}\bw +   \eta \bv_t^T\bh^{(t)}   - \frac{\|\bv_t\|_2^2}{4}\Bigr)  - \blambda ^T \bSigma^{1/2} (\bw - \be_i)
\end{align}
We defer further analyses of \eqref{eq: master_ao} to the Appendix due to the lack of space, but we wanted to illustrate that the proofs of Theorems \ref{thm: N=1} and \ref{thm: arb_N} proceed by a two-stage application of the CGMT followed by a meticuolous detailed analysis of \eqref{eq: master_ao}. 
\section{Experiments} \label{sec: experiments}
In this section we provide numerical experiments to corroborate our findings and discuss the proposed denoising algorithm.
\subsection{Verifying Corollary \ref{cor: scal} and Theorem \ref{thm: N=1}}
 We begin with verifying Corollary \ref{cor: scal}.  We take $\bSigma = \bI$ and $\bSigma_{\bz} = \bSigma_1 = c_{\bz} \bI$ and plot the performance of the linear denoiser against $\kappa$. We also use the closed-form expression obtained in \eqref{eq: opt_scal} to plot the optimal generalization error from Corollary \ref{cor: scal}.  Furthermore, we investigate how the error changes if we replace the Gaussian features with Rademacher with the same the first and second moments. The results are reported in Fig. \ref{fig:isos}, where cross marks denote the simulated values and solid-lines depict the  predictions coming from Corollary \ref{cor: scal}. We observe a close match between our predictions and the simulated errors for the Gaussian features as well as a close match between the generalization errors for Gaussian and Rademacher features. The latter suggests a form of universality, which we propose as a direction for future work. 
 
 To verify Theorem \ref{thm: N=1} for a non-scalar $\bSigma$, we generated $n$ synthetic data points of dimension $d = 50$ distributed i.i.d. according to $\calN(\bzero, \bSigma)$ where $\bSigma_{ij} = \delta_{ij}i^{-4}, i, j = 1,\dots,d$ (the power law exponent $4$ is chosen arbitrarily without any specific consideration in mind). For simplicity, we set $\bSigma_{\bz} = \bSigma_1 = \bI$. We vary the number of samples $n$ by changing $\kappa$ and train the denoiser via solving the least-squares objective through the linear system of equations by Karush–Kuhn–Tucker (KKT) conditions. For each $\kappa$, we evaluate the generalization error of the trained denoiser and compare it with our predictions from Theorem \ref{thm: N=1}. Finally, we also approximated the optimal $\bSigma_1$ dictated by the expression \eqref{eq: gen_thm1} from Theorem \ref{thm: N=1} via the following heuristics: we initialize $\bSigma_1 = \bSigma_{\bz}$ and then performed projected gradient descent on $\bSigma_1$. To make sure $\bSigma_1$ stays PSD, we fixed the basis of $\bSigma_1$ to be the basis of $\bSigma$ and trained vector of the eigenvalues  of $\bSigma_1$ via gradient descent with projection on the positive orthant.  The results are reported in Fig. \ref{fig:verify_derivation} where cross marks denote the simulated values and solid-lines depict the  predictions coming from Theorem \ref{thm: N=1}. We observe a close match between the theory and the simulated values as well as that optimizing over $\bSigma_1$ does increase performance. 

\subsection{Heuristics for using the result of Theorem \ref{thm: N=1} for training denoisers}

 Since, in practice, $\bSigma$ is not known, we propose to use the sample covariance \eqref{eq: emp_cov} and perform the following optimization as a surrogate of the expression \eqref{eq: gen_thm1} from Theorem \ref{thm: N=1}:
\begin{align}\label{eq: emp_opt_Sig1}
    \min_{\bSigma_1 \succeq \bzero} \Bigl(1 + \frac{1}{n-d} \tr \Bigl( \Bigl(\hat{\bSigma} + \bSigma_1 \Bigr)^{-1} \Bigl(\hat{\bSigma} + \bSigma_{\bz}\Bigr)  \Bigr) \Bigr) \Bigl( \tr \Bigl(\hat{\bSigma}\Bigr) - \tr \Bigl(\Bigl(\hat{\bSigma} + \bSigma_1\Bigr)^{-1} \hat{\bSigma}^2 \Bigr)\Bigr)\Bigr) 
\end{align}
Since the problem \eqref{eq: emp_opt_Sig1} is not convex in $\bSigma_1$, solving it precisely appears out of reach, so we suggest the following heuristics. We initialize $\bSigma_1 = \bSigma_{\bz}$ and, to make sure $\bSigma_1$ stays PSD, we fix the basis of $\bSigma_1$ to be the basis of $\hat{\bSigma}$ and train the vector of the eigenvalues  of $\bSigma_1$ via gradient descent with projection on the positive orthant. The performance of the true Wiener filter is depicted using dashed-lines and serves as a baseline for performance. 
To demonstrate that it is possible to outperform the empirical Wiener filter, we consider the case where $\bSigma_{\bz} = \bI$ and $\bSigma = c^2 \bG\bG^T$, with $c > 0$ being the scaling factor and $\bG \in \bbR^{d \times d}$ having i.i.d. standard normal entries. We choose a non-diagonal $\bSigma$ to ensure that it cannot be estimated reliably from $\Theta(d)$ samples. We vary $c$ and plot the performance of the empirical filter against the performance of  $\bW_{\text{lsq}}$ with optimized $\bSigma_1$ in Figure \ref{fig:ours}. We also have included the performance of the Wiener filter and the $\bW = \bI$ denoiser. As we increase $c$, we observe that performance of the empirical and true Wiener filters and our proposed filter approaches that of the identity denoiser. We see that, for large values of $c$, i.e. the high SNR regime, $\bW_{\text{lsq}}$ outperforms the empirical Wiener filter. 

\begin{figure*}[htb]
    \centering
    \begin{minipage}[t]{0.45\linewidth}
        \centering
        \includegraphics[width=\linewidth]{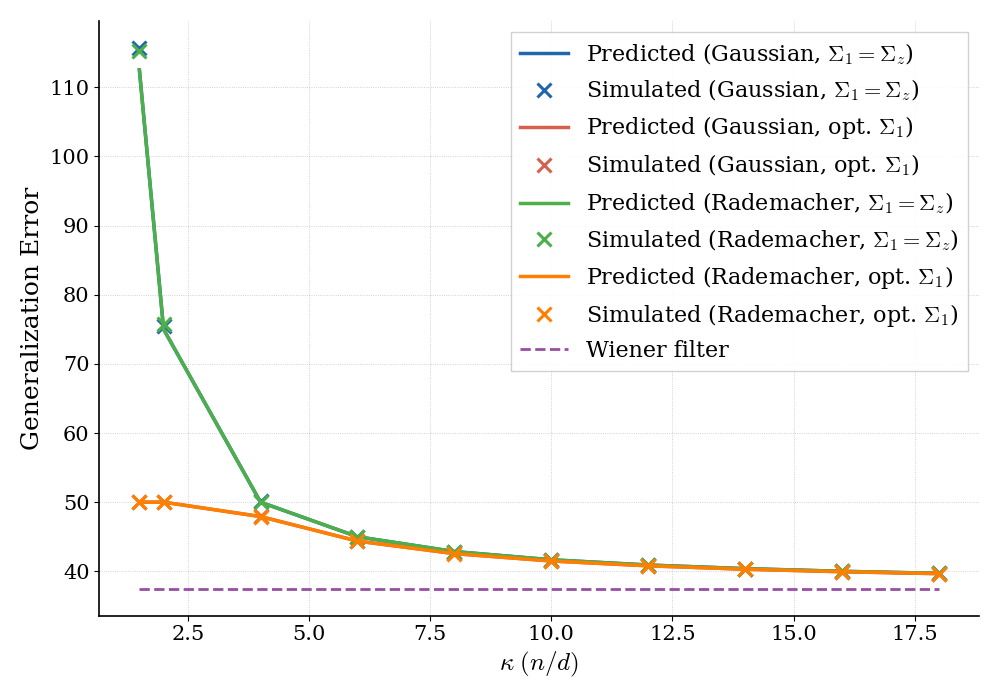}
        \textbf{(a)} $c_{\bz} = 3$
        \label{fig:iso1}
    \end{minipage}%
    \hfill
    \begin{minipage}[t]{0.45\linewidth}
        \centering
        \includegraphics[width=\linewidth]{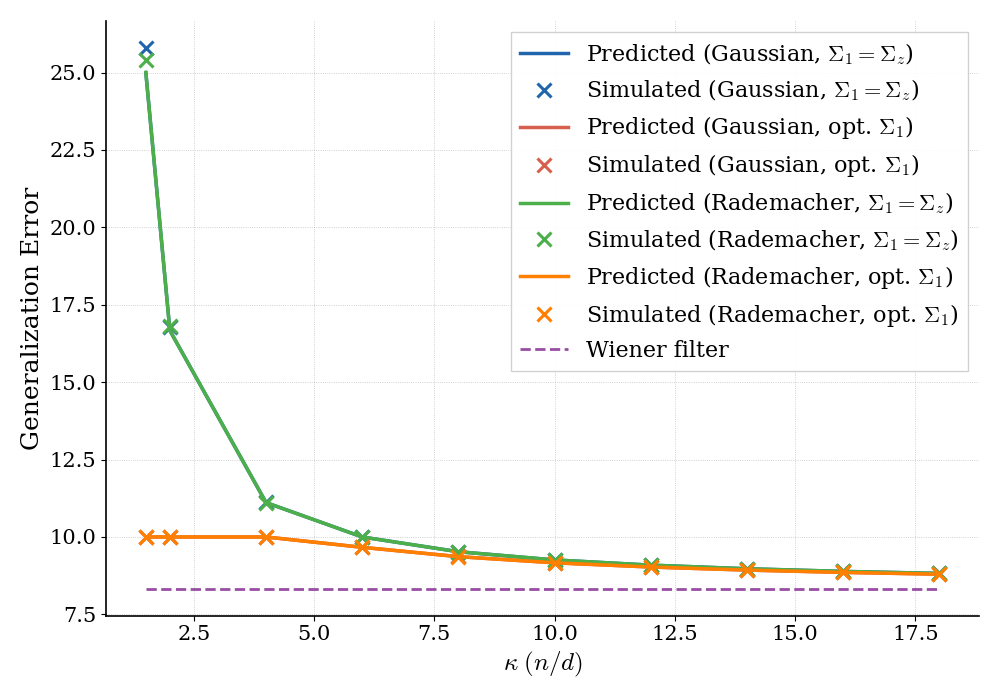}
        \textbf{(b)} $c_{\bz} = 0.2$
        \label{fig:iso2}
    \end{minipage}
    \caption{Verification of Corollary \ref{cor: scal} with $c = 1$}
    \label{fig:isos}
\end{figure*}
\begin{figure}[htb]
\centering\includegraphics[width=0.5\linewidth]{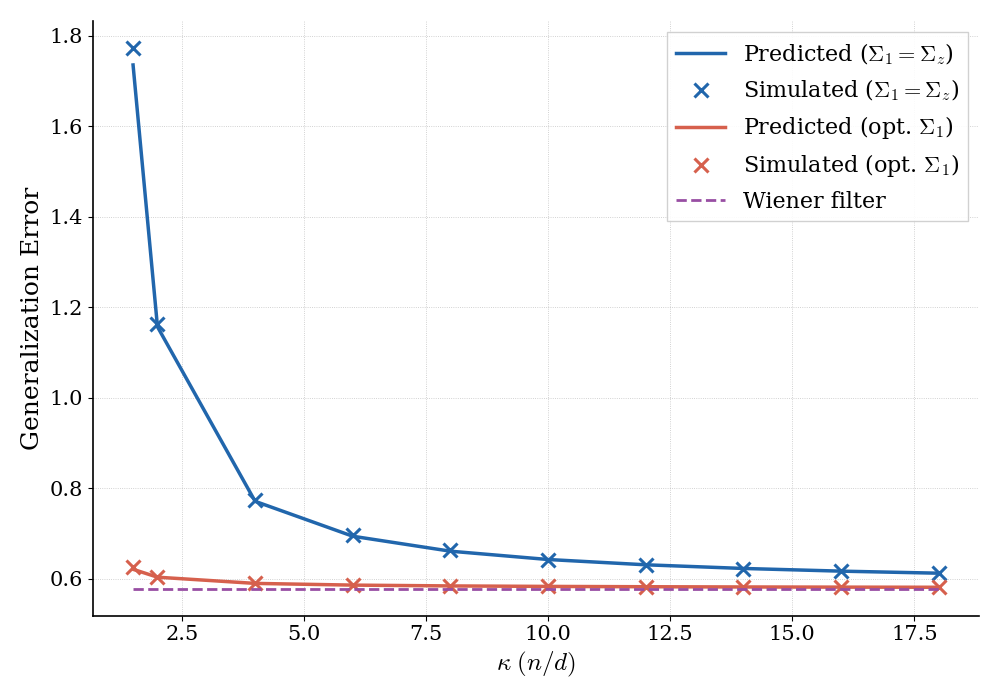}
    \caption{Numerical verification of Theorem \ref{thm: N=1}}
    \label{fig:verify_derivation}
\end{figure}


\begin{figure}[htb]
    \begin{minipage}[t]{0.45\linewidth}
        \centering
        \includegraphics[width=\linewidth]{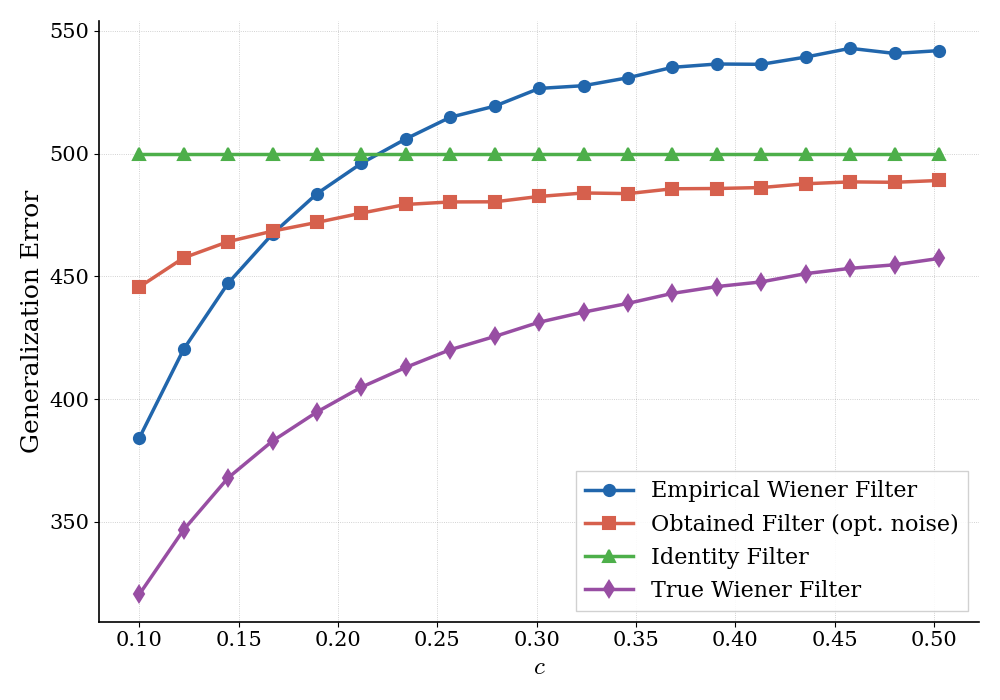}
        \textbf{(a)} $n=800$ and $d=500$
    \end{minipage}
    \hfill
    \begin{minipage}[t]{0.45\linewidth}
        \centering
        \includegraphics[width=\linewidth]{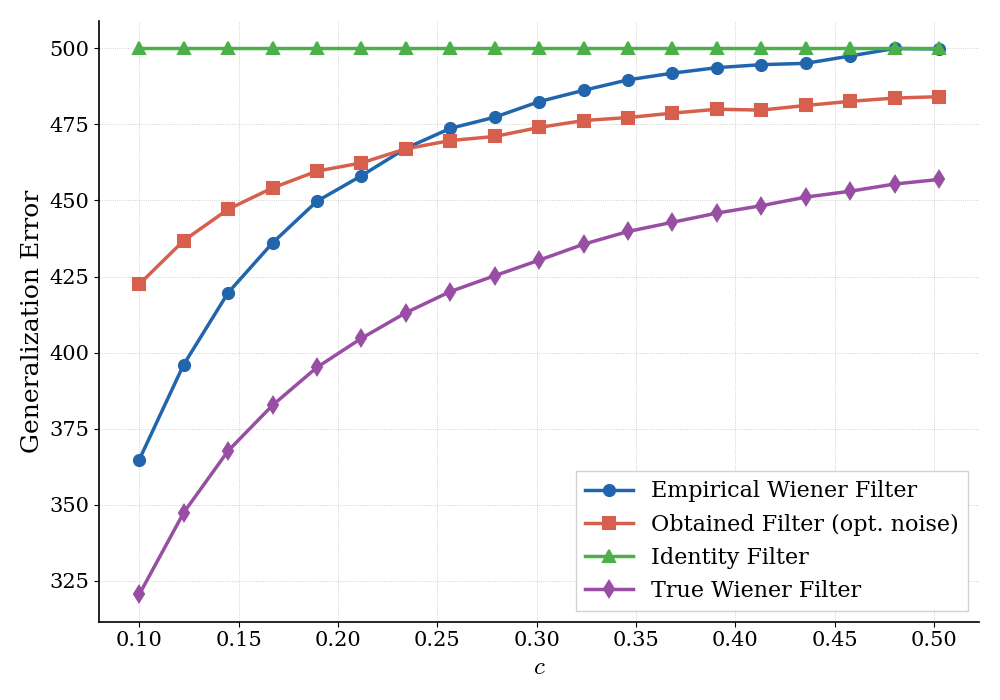}
        \textbf{(b)} $n=1000$ and $d=500$
    \end{minipage}
    \caption{Comparison of our approach to the empirical Wiener and true Wiener filters}
    \label{fig:ours}
\end{figure}

\section{Conclusion and Future Directions}

In this work we rigorously derived the precise asymptotics of a linear denoiser trained using the least-squares objective in the proportional regime. Possible future directions include investigating the expressions obtained in Theorem \ref{thm: arb_N} to understand how using different batches of data may help with the performance. Another promising direction is to extend the analysis to the overparameterized regime where the number of parameters exceed the number of samples and the denoiser is trained via SGD or via a regularized objective.

While we maintained the assumption of data gaussianity throughout the present work, we suspect it is not necessary for Theorems \ref{thm: N=1} and \ref{thm: arb_N}. As has been observed in the literature for many statistical inference problems, the generalization error often depends only on the first and second order statistics of the data, a phenomenon commonly referred to as Gaussian Universality (see \citep{pesce2023gaussian, ghane2024universality} and references therein). In addition, Figure \ref{fig:isos} suggests that a form of universality might hold for the denoisers trained through the least-squares objective too. We leave the latter as an interesting direction for future work. 
\newpage

\acks{R.G. and D.A. would like to thank Anthony Bao for many stimulating conversations on diffusion models that we have had together.} 

\bibliography{l4dc2026-sample}
\newpage

\appendix

\section{Proof of Theorems \ref{thm: N=1} and \ref{thm: arb_N}: common part}

Here we describe the proof in full detail. We start with the objective \eqref{eq: master_ao} obtained in the main body and take $i = 1$ WLOG:
\begin{align}
    \min_{\bw} \max_{\bv_t}  \frac{\beta}{2\tau} \Bigl( \sum_{t=1}^N \|\bSigma_t^{1/2}\bw\|_2 \bh_t^T \bv_t + \|\bv_t\|_2 \bg_t^T \bSigma_t^{1/2}\bw +   \eta \bv_t^T\bh^{(t)}   - \frac{\|\bv_t\|_2^2}{4}\Bigr)  - \blambda ^T \bSigma^{1/2} (\bw - \be_1)
\end{align}
Doing the optimization over the direction of $\bv_t$ yields:
\begin{align*}
    \min_{\bw} \max_{\beta_t> 0}  \frac{\beta}{2\tau} \Bigl( \sum_{t=1}^N \beta_t \Bigl\| \eta \bh^{(t)} +  \|\bSigma_t^{1/2}\bw\|_2 \bh_t \Bigr\|_2 + \beta_t \bg_t^T \bSigma_t^{1/2}\bw - \frac{\beta_t^2}{4} \Bigr)  - \blambda ^T \bSigma^{1/2} (\bw - \be_1)
\end{align*}
Using the square-root trick, we introduce the variables $\tau_t$ for $1\le t \le N$:
\begin{align*}
    \min_{\bw} \max_{\beta_t>0} \min_{\tau_t>0} &\frac{\beta}{2\tau} \Bigl( \sum_{t=1}^N \frac{\beta_t \tau_t}{2} + \frac{\beta_t}{2\tau_t} \Bigl( \eta^2 n_t +  \|\bSigma_t^{1/2}\bw\|^2_2 n_t \Bigr) + \beta_t \bg_t^T \bSigma_t^{1/2}\bw - \frac{\beta_t^2}{4}\Bigr) \\ 
    &- \blambda ^T \bSigma^{1/2} (\bw - \be_1)
\end{align*}
Using the Sion's minimax theorem and convex-concavity of objective, we exchange the order of $\min_{\bw}$ with $\max_{\beta_t>0} \min_{\tau_t>0}$. Now we observe that the optimization over $\bw$ is quadratic, after performing the optimization over $\bw$, we obtain:
\begin{align*}
    \max_{\blambda} \min_{\eta > 0} \max_{\beta>0}  &-\eta  \|\beta \bg + \blambda\|_2 + \blambda ^T \bSigma^{1/2} \be_1 - \frac{\beta^2}{4} + \min_{ \tau>0} \frac{\beta \tau}{2}\\
    &+\max_{\beta_t>0} \min_{\tau_t>0} \frac{\beta}{2\tau} \Bigl( \sum_{t=1}^N \frac{\beta_t \tau_t}{2} + \frac{\beta_t}{2\tau_t} \eta^2 n_t - \frac{\beta_t^2}{4} \Bigr)\\
    & -\frac{1}{4}\Bigl(-\bSigma^{1/2} \blambda + \frac{\beta}{2\tau}\sum_{t=1}^N \beta_t \bSigma_t^{1/2}\bg_t\Bigr)^T \Bigl(\frac{\beta}{2\tau}\sum_{t=1}^N \frac{\beta_t}{2\tau_t} n_t \bSigma_t \Bigr)^{-1} \\
    &\cdot \Bigl(-\bSigma^{1/2} \blambda + \frac{\beta}{2\tau}\sum_{t=1}^N \beta_t \bSigma_t^{1/2}\bg_t\Bigr)
\end{align*}
Where the optimal $\bw$ can be written as
\begin{align*}
    \bw = -\frac{1}{2}\Bigl(\frac{\beta}{2\tau}\sum_{t=1}^N \frac{\beta_t}{2\tau_t} n_t \bSigma_t \Bigr)^{-1}\Bigl(-\bSigma^{1/2} \blambda + \frac{\beta}{2\tau}\sum_{t=1}^N \beta_t \bSigma_t^{1/2}\bg_t\Bigr)
\end{align*}
Hence, we need to find $\blambda \in \bbR^{d}$. Employing the square-root trick on the first term:
\begin{align*}
    \max_{\blambda}\max_{\tau_{\lambda}>0} \min_{\eta > 0} \max_{\beta>0} & -\frac{\eta \tau_{\lambda}}{2} - \frac{\eta}{2\tau_{\lambda}}  \|\beta \bg + \blambda\|_2^2 +\blambda ^T \bSigma^{1/2} \be_1 - \frac{\beta^2}{4} + \min_{ \tau>0} \frac{\beta \tau}{2}\\
    &+\max_{\beta_t>0} \min_{\tau_t>0} \frac{\beta}{2\tau} \Bigl( \sum_{t=1}^N \frac{\beta_t \tau_t}{2} + \frac{\beta_t}{2\tau_t} \eta^2 n_t - \frac{\beta_t^2}{4}\Bigr)\\
     &-\frac{1}{4}\Bigl(-\bSigma^{1/2} \blambda + \frac{\beta}{2\tau}\sum_{t=1}^N \beta_t \bSigma_t^{1/2}\bg_t\Bigr)^T \Bigl(\frac{\beta}{2\tau}\sum_{t=1}^N \frac{\beta_t}{2\tau_t} n_t \bSigma_t \Bigr)^{-1}\\
     &\cdot\Bigl(-\bSigma^{1/2} \blambda + \frac{\beta}{2\tau}\sum_{t=1}^N \beta_t \bSigma_t^{1/2}\bg_t\Bigr)
\end{align*}
The optimal $\blambda$ is
\begin{align*}
    \blambda &:= \Bigl(\frac{\eta}{2\tau_{\lambda}}\bI+\frac{1}{4}\bSigma^{1/2} \Bigl(\frac{\beta}{2\tau}\sum_{t=1}^N \frac{\beta_t}{2\tau_t} n_t \bSigma_t \Bigr)^{-1}\bSigma^{1/2}\Bigr)^{-1} \\ 
    &\cdot \biggl[-\frac{\eta \beta}{2\tau_{\lambda}} \bg + \frac{1}{4} \bSigma^{1/2} \Bigl(\frac{\beta}{2\tau}\sum_{t=1}^N \frac{\beta_t}{2\tau_t} n_t \bSigma_t \Bigr)^{-1} \frac{\beta}{2\tau}\sum_{t=1}^N \beta_t \bSigma_t^{1/2}\bg_t + \frac{1}{2} \bSigma^{1/2} \be_1 \biggr]
\end{align*}
This implies the optimal $\bw$ is
\begin{align*}
    &\Bigl(\frac{\beta}{2\tau}\sum_{t=1}^N \frac{\beta_t}{2\tau_t} n_t \bSigma_t \Bigr)^{-1} \bSigma^{1/2}\Bigl(\frac{\eta}{2\tau_{\lambda}}\bI+\frac{1}{4}\bSigma^{1/2} \Bigl(\frac{\beta}{2\tau}\sum_{t=1}^N \frac{\beta_t}{2\tau_t} n_t \bSigma_t \Bigr)^{-1}\bSigma^{1/2}\Bigr)^{-1} \frac{\eta \beta}{2\tau_{\lambda}} \bg \\
    &- \Bigl(\frac{\beta}{2\tau}\sum_{t=1}^N \frac{\beta_t}{2\tau_t} n_t \bSigma_t \Bigr)^{-1} \bSigma^{1/2} \Bigl(\frac{\eta}{2\tau_{\lambda}}\bI+\frac{1}{4}\bSigma^{1/2} \Bigl(\frac{\beta}{2\tau}\sum_{t=1}^N \frac{\beta_t}{2\tau_t} n_t \bSigma_t \Bigr)^{-1}\bSigma^{1/2}\Bigr)^{-1} \frac{1}{2} \bSigma^{1/2} \be_1 \\
    &- \Bigl(\frac{\beta}{2\tau}\sum_{t=1}^N \frac{\beta_t}{2\tau_t} n_t \bSigma_t \Bigr)^{-1} \bSigma^{1/2} \Bigl(\frac{\eta}{2\tau_{\lambda}}\bI+\frac{1}{4}\bSigma^{1/2} \Bigl(\frac{\beta}{2\tau}\sum_{t=1}^N \frac{\beta_t}{2\tau_t} n_t \bSigma_t \Bigr)^{-1}\bSigma^{1/2}\Bigr)^{-1} \\ 
    & \cdot \frac{1}{4}\bSigma^{1/2} \Bigl(\frac{\beta}{2\tau}\sum_{t=1}^N \frac{\beta_t}{2\tau_t} n_t \bSigma_t \Bigr)^{-1} \frac{\beta}{2\tau}\sum_{t=1}^N \beta_t \bSigma_t^{1/2}\bg_t \\
    &+\Bigl(\frac{\beta}{2\tau}\sum_{t=1}^N \frac{\beta_t}{2\tau_t} n_t \bSigma_t \Bigr)^{-1}\frac{\beta}{2\tau}\sum_{t=1}^N \beta_t \bSigma_t^{1/2}\bg_t
\end{align*}
Using matrix inversion lemma
\begin{align*}
    \bw^\ast &= -\frac{1}{2}\biggl[4\beta \Bigl(  \frac{\beta}{\tau}\sum_{t=1}^N \frac{\beta_t}{\tau_t} n_t \bSigma_t + \frac{2\tau_{\lambda}}{\eta} \bSigma \Bigr)^{-1} \bSigma^{1/2} \bg -  \frac{4\tau_{\lambda}}{\eta} \Bigl(  \frac{\beta}{\tau}\sum_{t=1}^N \frac{\beta_t}{\tau_t} n_t \bSigma_t + \frac{2\tau_{\lambda}}{\eta} \bSigma \Bigr)^{-1} \bSigma \be_1\\
    &+ 4 \Bigl(  \frac{\beta}{\tau}\sum_{t=1}^N \frac{\beta_t}{\tau_t} n_t \bSigma_t + \frac{2\tau_{\lambda}}{\eta} \bSigma \Bigr)^{-1} \frac{\beta}{2\tau}\sum_{t=1}^N \beta_t \bSigma_t^{1/2}\bg_t \biggr] \\
    &= \Bigl(  \frac{\beta}{\tau}\sum_{t=1}^N \frac{\beta_t}{\tau_t} n_t \bSigma_t + \frac{2\tau_{\lambda}}{\eta} \bSigma \Bigr)^{-1} \Bigl( - 2 \beta \bSigma^{1/2} \bg + \frac{2 \tau_{\lambda}}{\eta} \bSigma \be_1 -  \frac{\beta}{\tau} \sum_{t=1}^N \beta_t \bSigma_t^{1/2}\bg_t \Bigr)
\end{align*}
Performing the optimization over $\blambda$ yields the following scalar optimization:
\begin{align*}
    \max_{\tau_{\lambda}>0} \min_{\eta > 0} \max_{\beta>0} & -\frac{\eta \tau_{\lambda}}{2} - \frac{\beta^2}{4} + \min_{ \tau>0} \frac{\beta \tau}{2}+\max_{\beta_t} \min_{\tau_t} \frac{\beta}{2\tau} \Bigl( \sum_{t=1}^N \frac{\beta_t \tau_t}{2} + \frac{\beta_t}{2\tau_t} \eta^2 n_t - \frac{\beta_t^2}{4}\Bigr)\\
     &-\beta^2\bg^T \Bigl(\frac{2\tau_{\lambda}}{\eta}\bI+\frac{\beta}{\tau}\sum_{t=1}^N \frac{\beta_t}{\tau_t} n_t \bSigma^{-1/2}\bSigma_t \bSigma^{-1/2} \Bigr)^{-1} \bg\\ 
      &-\frac{\beta^2}{4\tau^2} \Bigl(\sum_{t=1}^N \beta_t \bSigma^{-1/2} \bSigma_t^{1/2}\bg_t\Bigr)^T\Bigl(\frac{2\tau_{\lambda}}{\eta}\bI+\frac{\beta}{\tau}\sum_{t=1}^N \frac{\beta_t}{\tau_t} n_t \bSigma^{-1/2}\bSigma_t \bSigma^{-1/2} \Bigr)^{-1} \\
      &\cdot \sum_{t=1}^N \beta_t \bSigma^{-1/2} \bSigma_t^{1/2}\bg_t\\ 
     &+\frac{1}{4}\be_1^T\bSigma^{1/2}\Bigl(\frac{\eta}{2\tau_{\lambda}}\bI+\frac{1}{4}\bSigma^{1/2} \Bigl(\frac{\beta}{2\tau}\sum_{t=1}^N \frac{\beta_t}{2\tau_t} n_t \bSigma_t \Bigr)^{-1}\bSigma^{1/2}\Bigr)^{-1}\bSigma^{1/2} \be_1
\end{align*}
We apply the matrix inversion lemma
\begin{align*}
    &\be_1^T\bSigma^{1/2}\Bigl(\frac{\eta}{2\tau_{\lambda}}\bI+\frac{1}{4}\bSigma^{1/2} \Bigl(\frac{\beta}{2\tau}\sum_{t=1}^N \frac{\beta_t}{2\tau_t} n_t \bSigma_t \Bigr)^{-1}\bSigma^{1/2}\Bigr)^{-1}\bSigma^{1/2} \be_1\\
    &= \frac{2\tau_{\lambda}}{\eta} \be_1^T \bSigma \be_1
    - \frac{4 \tau^2_{\lambda}}{\eta^2} \be_1^T \bSigma^{1/2} \Bigl(\frac{2\tau_{\lambda}}{\eta}\bI+\frac{\beta}{\tau}\sum_{t=1}^N \frac{\beta_t}{\tau_t} n_t \bSigma^{-1/2}\bSigma_t \bSigma^{-1/2} \Bigr)^{-1} \bSigma^{1/2} \be_1
\end{align*}
By Hanson-Wright inequality, we observe that:
\begin{align*}
    &\bg^T \Bigl(\frac{2\tau_{\lambda}}{\eta}\bI+\frac{\beta}{\tau}\sum_{t=1}^N \frac{\beta_t}{\tau_t} n_t \bSigma^{-1/2}\bSigma_t \bSigma^{-1/2} \Bigr)^{-1} \bg 
    \rarrowp \tr \Bigl(\frac{2\tau_{\lambda}}{\eta}\bI+\frac{\beta}{\tau}\sum_{t=1}^N \frac{\beta_t}{\tau_t} n_t \bSigma^{-1/2}\bSigma_t \bSigma^{-1/2} \Bigr)^{-1} \\
    &\Bigl(\sum_{t=1}^N \beta_t \bSigma^{-1/2} \bSigma_t^{1/2}\bg_t\Bigr)^T\Bigl(\frac{2\tau_{\lambda}}{\eta}\bI+\frac{\beta}{\tau}\sum_{t=1}^N \frac{\beta_t}{\tau_t} n_t \bSigma^{-1/2}\bSigma_t \bSigma^{-1/2} \Bigr)^{-1}\sum_{t=1}^N \beta_t \bSigma^{-1/2} \bSigma_t^{1/2}\bg_t \\
    &\rarrowp  \tr  \biggl[\Bigl(\frac{2\tau_{\lambda}}{\eta}\bI+\frac{\beta}{\tau}\sum_{t=1}^N \frac{\beta_t}{\tau_t} n_t \bSigma^{-1/2}\bSigma_t \bSigma^{-1/2} \Bigr)^{-1} \sum_{t=1}^N \beta^2_t \bSigma^{-1/2}  \bSigma_t \bSigma^{-1/2}\biggr]
\end{align*}
Thus the final scalar optimization turns into
\begin{align} \label{opt: scal}
    \max_{\tau_{\lambda}>0} \min_{\eta > 0} \max_{\beta>0}  \min_{ \tau>0} \max_{\beta_t>0} &\min_{\tau_t>0}  -\frac{\eta \tau_{\lambda}}{2} - \frac{\beta^2}{4} + \frac{\beta \tau}{2}+ \frac{\beta}{2\tau} \Bigl( \sum_{t=1}^N \frac{\beta_t \tau_t}{2} + \frac{\beta_t}{2\tau_t} \eta^2 n_t - \frac{\beta_t^2}{4} \Bigr) +  \frac{\tau_{\lambda}}{2\eta} \be_1^T \bSigma \be_1  \\
     &-\tr \biggl[\Bigl(\frac{2\tau_{\lambda}}{\eta}\bSigma+\frac{\beta}{\tau}\sum_{t=1}^N \frac{\beta_t}{\tau_t} n_t \bSigma_t \Bigr)^{-1} \Bigl(\beta^2 \bSigma + \frac{\beta^2}{4\tau^2} \sum_{t=1}^N \beta^2_t  \bSigma_t  + \frac{\tau_{\lambda}^2}{\eta^2}\bSigma \be_1 \be_1^T \bSigma\Bigr) \biggr]
\end{align}
We define
\begin{align*}
    F\Bigl(\frac{\beta}{\tau}, \frac{\beta_t}{\tau_t}, \frac{\eta}{\tau_{\lambda}}, \beta, \beta_t\Bigr) := \tr \biggl[\Bigl(\frac{2\tau_{\lambda}}{\eta}\bSigma+\frac{\beta}{\tau}\sum_{t=1}^N \frac{\beta_t}{\tau_t} n_t \bSigma_t \Bigr)^{-1} \Bigl(\beta^2 \bSigma + \frac{\beta^2}{4\tau^2} \sum_{t=1}^N \beta^2_t  \bSigma_t  + \frac{\tau_{\lambda}^2}{\eta^2}\bSigma \be_1 \be_1^T \bSigma\Bigr) \biggr]
\end{align*}
Now we compute the stationary coditions for the scalar optimization in \eqref{opt: scal}. For the derivatives w.r.t $\beta$ and $\tau$
\begin{align*}
    \frac{\partial}{\partial \beta} &= 0 \Rightarrow 0 =- \frac{\beta}{2} + \frac{\tau}{2} + \frac{1}{2\tau} \Bigl( \sum_{t=1}^N \frac{\beta_t \tau_t}{2} + \frac{\beta_t}{2\tau_t} \eta^2 n_t - \frac{\beta_t^2}{4}\Bigr) \\
    &- \frac{1}{\tau} \partial_1  F\Bigl(\frac{\beta}{\tau}, \frac{\beta_t}{\tau_t}, \frac{\eta}{\tau_{\lambda}}, \beta, \beta_t\Bigr)  -  2\beta \tr \biggl[\Bigl(\frac{2\tau_{\lambda}}{\eta}\bSigma+\frac{\beta}{\tau}\sum_{t=1}^N \frac{\beta_t}{\tau_t} n_t \bSigma_t \Bigr)^{-1} \bSigma \biggr]\\
     \frac{\partial}{\partial \tau} &= 0 \Rightarrow 0 = \frac{\beta}{2} - \frac{\beta}{2\tau^2} \Bigl( \sum_{t=1}^N \frac{\beta_t \tau_t}{2} + \frac{\beta_t}{2\tau_t} \eta^2 n_t - \frac{\beta_t^2}{4} \Bigr)
     + \frac{\beta}{\tau^2} \partial_1  F\Bigl(\frac{\beta}{\tau}, \frac{\beta_t}{\tau_t}, \frac{\eta}{\tau_{\lambda}}, \beta, \beta_t\Bigr) 
\end{align*}
Thus for $\theta:=\frac{\beta}{\tau}$ we have
\begin{align*}
    \frac{\theta}{2}  &= 1 - 2 \theta \tr \Bigl(\frac{2\tau_{\lambda}}{\eta}\bSigma+\theta\sum_{t=1}^N \frac{\beta_t}{\tau_t} n_t \bSigma_t \Bigr)^{-1} \bSigma \\
    \tau^2 &= \sum_{t=1}^N \frac{\beta_t \tau_t}{2} + \frac{\beta_t}{2\tau_t} \eta^2 n_t - \frac{\beta_t^2}{4} - 2 \partial_1  F\Bigl(\frac{\beta}{\tau}, \frac{\beta_t}{\tau_t}, \frac{\eta}{\tau_{\lambda}}, \beta, \beta_t\Bigr)
\end{align*}
Where
\begin{align*}
    \partial_1  F\Bigl(\frac{\beta}{\tau}, \frac{\beta_t}{\tau_t}, \frac{\eta}{\tau_{\lambda}}, \beta, \beta_t\Bigr) &= -\tr \biggl[ \Bigl(\frac{2\tau_{\lambda}}{\eta}\bSigma+\frac{\beta}{\tau}\sum_{t=1}^N \frac{\beta_t}{\tau_t} n_t \bSigma_t \Bigr)^{-1} \Bigl(\sum_{t=1}^N \frac{\beta_t}{\tau_t} n_t \bSigma_t \Bigr) \\
    &\cdot\Bigl(\frac{2\tau_{\lambda}}{\eta}\bSigma+\frac{\beta}{\tau}\sum_{t=1}^N \frac{\beta_t}{\tau_t} n_t \bSigma_t \Bigr)^{-1} \Bigl(\beta^2 \bSigma + \frac{\beta^2}{4\tau^2} \sum_{t=1}^N \beta^2_t  \bSigma_t  + \frac{\tau_{\lambda}^2}{\eta^2}\bSigma \be_1 \be_1^T \bSigma\Bigr) \biggr]
    \\
    &+ \frac{\beta}{2\tau} \tr\biggl[ \Bigl(\frac{2\tau_{\lambda}}{\eta}\bSigma+\frac{\beta}{\tau}\sum_{t=1}^N \frac{\beta_t}{\tau_t} n_t \bSigma_t \Bigr)^{-1}  \sum_{t=1}^N \beta^2_t  \bSigma_t \biggr]
\end{align*}
Similarly for $\eta$ and $\tau_{\lambda}$
\begin{align*}
     \frac{\partial}{\partial \eta} &= 0 \Rightarrow 0 = - \frac{\tau_{\lambda}}{2} + \frac{\beta  \eta n_t }{\tau} \sum_{t=1}^N \frac{\beta_t}{2\tau_t} - \frac{\tau_{\lambda}}{2\eta^2} \be_1^T \bSigma \be_1 -\frac{1}{\tau_{\lambda}} \partial_3  F\Bigl(\frac{\beta}{\tau}, \frac{\beta_t}{\tau_t}, \frac{\eta}{\tau_{\lambda}}, \beta, \beta_t\Bigr)  \\
      \frac{\partial}{\partial \tau_{\lambda}} &= 0 \Rightarrow 0 = -\frac{\eta}{2} + \frac{1}{2\eta} \be_1^T \bSigma \be_1
     + \frac{\eta}{\tau_{\lambda}^2} \partial_3  F\Bigl(\frac{\beta}{\tau}, \frac{\beta_t}{\tau_t}, \frac{\eta}{\tau_{\lambda}}, \beta, \beta_t\Bigr) 
\end{align*}
Thus for $\zeta:= \frac{\eta}{\tau_{\lambda}}$ we have
\begin{align*}
    \frac{1}{\zeta} &=  \theta n_t \sum_{t=1}^N \frac{\beta_t}{2\tau_t} \\
    \tau^2_{\lambda} &= \frac{1}{\zeta^2} \be_1^T \bSigma \be_1 + 2\partial_3 F\Bigl(\frac{\beta}{\tau}, \frac{\beta_t}{\tau_t}, \frac{\eta}{\tau_{\lambda}}, \beta, \beta_t\Bigr) 
\end{align*}
Where
\begin{align*}
    \partial_3 F\Bigl(\frac{\beta}{\tau}, \frac{\beta_t}{\tau_t}, \frac{\eta}{\tau_{\lambda}}, \beta, \beta_t\Bigr) &= \frac{2\tau_{\lambda}^2}{\eta^2} \tr \biggl[\Bigl(\frac{2\tau_{\lambda}}{\eta}\bSigma+\frac{\beta}{\tau}\sum_{t=1}^N \frac{\beta_t}{\tau_t} n_t \bSigma_t \Bigr)^{-1} \bSigma \\&\Bigl(\frac{2\tau_{\lambda}}{\eta}\bSigma+\frac{\beta}{\tau}\sum_{t=1}^N \frac{\beta_t}{\tau_t} n_t \bSigma_t \Bigr)^{-1} \Bigl(\beta^2 \bSigma + \frac{\beta^2}{4\tau^2} \sum_{t=1}^N \beta^2_t  \bSigma_t  + \frac{\tau_{\lambda}^2}{\eta^2}\bSigma \be_1 \be_1^T \bSigma\Bigr) \biggr] \\
    &- 2\frac{\tau_{\lambda}^3}{\eta^3} \tr \biggl[\Bigl(\frac{2\tau_{\lambda}}{\eta}\bSigma+\frac{\beta}{\tau}\sum_{t=1}^N \frac{\beta_t}{\tau_t} n_t \bSigma_t \Bigr)^{-1} \bSigma^{1/2} \be_1 \be_1 \bSigma^{1/2}\biggr]
\end{align*}
Similarly for $\beta_t$ and $\tau_t$
\begin{align*}
     \frac{\partial}{\partial \beta_t} = 0 \Rightarrow 0 &= \frac{\beta}{2\tau} \Bigl(\frac{\tau_t}{2} + \frac{\eta^2 n_t}{2\tau_t} - \frac{\beta_t}{2}\Bigr) - \frac{1}{\tau_t} \partial_2 F\Bigl(\frac{\beta}{\tau}, \frac{\beta_t}{\tau_t}, \frac{\eta}{\tau_{\lambda}}, \beta, \beta_t\Bigr) \\
     &- \frac{\beta^2 \beta_t}{2\tau^2}  \tr\biggl[ \Bigl(\frac{2\tau_{\lambda}}{\eta}\bSigma+\frac{\beta}{\tau}\sum_{t=1}^N \frac{\beta_t}{\tau_t} n_t \bSigma_t \Bigr)^{-1} \bSigma_t \biggr]\\
      \frac{\partial}{\partial \tau_t} = 0 \Rightarrow 0 &= \frac{\beta}{2\tau} \Bigl(\frac{\beta_t}{2} - \frac{\beta_t \eta^2 n_t}{2\tau_t^2} \Bigr) +  \frac{\beta_t}{\tau^2_t} \partial_2 F\Bigl(\frac{\beta}{\tau}, \frac{\beta_t}{\tau_t}, \frac{\eta}{\tau_{\lambda}}, \beta, \beta_t\Bigr)
\end{align*}
This implies for $\theta_t := \frac{\beta_t}{\tau_t}$
\begin{align*}
    \theta_t &= 2 - 2 \theta \theta_t \tr\biggl[ \Bigl(\frac{2}{\zeta}\bSigma+\theta\sum_{t=1}^N \theta_t n_t \bSigma_t \Bigr)^{-1} \bSigma_t \biggr] \\
    \tau_t^2 &= \zeta^2 \tau_{\lambda}^2 n_t - \frac{4}{\theta} \partial_2 F\Bigl(\frac{\beta}{\tau}, \frac{\beta_t}{\tau_t}, \frac{\eta}{\tau_{\lambda}}, \beta, \beta_t\Bigr)
\end{align*}
Where
\begin{align*}
     \partial_2 F\Bigl(\frac{\beta}{\tau}, \frac{\beta_t}{\tau_t}, \frac{\eta}{\tau_{\lambda}}, \beta, \beta_t\Bigr) &= - \frac{\beta n_t}{\tau}\tr\biggl[ \Bigl(\frac{2\tau_{\lambda}}{\eta}\bSigma+\frac{\beta}{\tau}\sum_{t=1}^N \frac{\beta_t}{\tau_t} n_t \bSigma_t \Bigr)^{-1} \bSigma_t  \\
     &\cdot \Bigl(\frac{2\tau_{\lambda}}{\eta}\bSigma+\frac{\beta}{\tau}\sum_{t=1}^N \frac{\beta_t}{\tau_t} n_t \bSigma_t \Bigr)^{-1} \Bigl(\beta^2 \bSigma + \frac{\beta^2}{4\tau^2} \sum_{t=1}^N \beta^2_t  \bSigma_t  + \frac{\tau_{\lambda}^2}{\eta^2}\bSigma \be_1 \be_1^T \bSigma\Bigr) \biggr] 
\end{align*}
Now we summarize the stationary conditions as follows:
\begin{align}\label{eq: thm_arbN_eqns}
    \frac{\theta}{2}  &= 1 - 2 \theta \tr \biggl[\Bigl(\frac{2}{\zeta}\bSigma+\theta\sum_{t=1}^N \theta_t n_t \bSigma_t \Bigr)^{-1} \bSigma \biggr] \nonumber\\
    \frac{2}{\zeta} &= \theta n_t \sum_{t=1}^N \theta_t \nonumber\\
    \theta_t &= 2 - 2 \theta \theta_t \tr \biggl[\Bigl(\frac{2}{\zeta}\bSigma+\theta\sum_{t=1}^N \theta_t n_t \bSigma_t \Bigr)^{-1} \bSigma_t\biggr] \nonumber\\
    \tau^2 &= \sum_{t=1}^N \frac{\theta_t \tau^2_t}{2} + \frac{\theta_t}{2} \zeta^2 \tau_{\lambda}^2 n_t - \frac{\tau_t^2 \theta_t^2}{4} - 2 \partial_1  F\Bigl(\frac{\beta}{\tau}, \frac{\beta_t}{\tau_t}, \frac{\eta}{\tau_{\lambda}}, \beta, \beta_t\Bigr) \nonumber\\
    \tau_t^2 &= \zeta^2 \tau_{\lambda}^2 n_t - \frac{4}{\theta} \partial_2 F\Bigl(\theta, \theta_t, \zeta, \beta, \beta_t\Bigr)\nonumber \\
    \tau^2_{\lambda} &=\frac{1}{\zeta^2} \be_1^T \bSigma \be_1 + 2\partial_3 F\Bigl(\theta, \theta_t, \zeta, \beta, \beta_t\Bigr) \nonumber\\
    \partial_1  F\Bigl(\theta, \theta_t, \zeta, \beta, \beta_t\Bigr) &= -\tr \biggl[\Bigl(\frac{2}{\zeta}\bSigma+\theta\sum_{t=1}^N \theta_t n_t \bSigma_t \Bigr)^{-1} \Bigl(\sum_{t=1}^N \theta_t n_t \bSigma_t \Bigr) \Bigl(\frac{2}{\zeta}\bSigma+\theta\sum_{t=1}^N \theta_t n_t \bSigma_t \Bigr)^{-1}  \nonumber\\ 
    & \cdot \Bigl(\theta^2 \tau^2 \bSigma + \frac{\theta^2}{4} \sum_{t=1}^N \theta^2_t \tau_t^2  \bSigma_t  +\frac{1}{\zeta^2}\bSigma \be_1 \be_1^T \bSigma\Bigr) \biggr] \nonumber \\
    &+ \frac{\theta}{2} \tr\biggl[ \Bigl(\frac{2}{\zeta}\bSigma+\theta\sum_{t=1}^N \theta_t n_t \bSigma_t \Bigr)^{-1}  \sum_{t=1}^N \theta^2_t \tau_t^2  \bSigma_t\biggr] \nonumber\\
     \partial_2 F\Bigl(\theta, \theta_t, \zeta, \beta, \beta_t\Bigr) &= - \theta n_t\tr \biggl[\Bigl(\frac{2}{\zeta}\bSigma+\theta\sum_{t=1}^N \theta_t n_t \bSigma_t \Bigr)^{-1} \bSigma_t  \Bigl(\frac{2}{\zeta}\bSigma+\theta\sum_{t=1}^N \theta_t n_t \bSigma_t \Bigr)^{-1} \nonumber\\
     &\cdot \Bigl(\theta^2 \tau^2 \bSigma + \frac{\theta^2}{4} \sum_{t=1}^N \theta^2_t \tau_t^2  \bSigma_t  + \frac{1}{\zeta^2}\bSigma \be_1 \be_1^T \bSigma\Bigr) \biggr]\nonumber\\
      \partial_3 F\Bigl(\theta, \theta_t, \zeta, \beta, \beta_t\Bigr) &= \frac{2}{\zeta^2} \tr \biggl[ \Bigl(\frac{2}{\zeta}\bSigma+\theta\sum_{t=1}^N \theta_t n_t \bSigma_t \Bigr)^{-1} \bSigma  \Bigl(\frac{2}{\zeta}\bSigma+\theta\sum_{t=1}^N \theta_t n_t \bSigma_t \Bigr)^{-1} \nonumber\\
      &\cdot \Bigl(\theta^2 \tau^2 \bSigma + \frac{\theta^2}{4} \sum_{t=1}^N \theta^2_t \tau_t^2 \bSigma_t  + \frac{1}{\zeta^2}\bSigma \be_1 \be_1^T \bSigma\Bigr) \biggr] \nonumber \\
      &- \frac{2}{\zeta^3} \tr\bigg[\Bigl(\frac{2}{\zeta}\bSigma+\theta\sum_{t=1}^N \theta_t n_t \bSigma_t \Bigr)^{-1}  \bSigma \be_1 \be_1^T \bSigma\biggr]
\end{align}
At the optimal points, we have for $\bw$
\begin{align*}
    \bw^\ast =  \Bigl(  \theta \sum_{t=1}^N \theta_t n_t \bSigma_t + \frac{2}{\zeta} \bSigma \Bigr)^{-1} \Bigl( - 2 \theta \tau \bSigma^{1/2} \bg + 2 \bSigma \be_1 -  \theta \sum_{t=1}^N \theta_t \tau_t \bSigma_t^{1/2}\bg_t \Bigr)
\end{align*}
The generalization error corresponding to this term ($i=1$) is
\begin{align*}
    \eta^2 &+ \bw^T \bSigma_{\bz} \bw  \\
    \rarrowp \eta^2  &+ \tr \biggl[ \Bigl(  \theta \sum_{t=1}^N \theta_t n_t \bSigma_t + \frac{2}{\zeta} \bSigma \Bigr)^{-1} \bSigma_{\bz}  \Bigl(  \theta \sum_{t=1}^N \theta_t n_t \bSigma_t + \frac{2}{\zeta} \bSigma \Bigr)^{-1} \\
    &\cdot \Bigl( 4 \theta^2 \tau^2 \bSigma + \frac{4}{\zeta^2} \bSigma \be_1 \be_1^T \bSigma + \theta^2 \sum_{t=1}^N \theta^2_t \tau^2_t \bSigma_t \Bigr)\biggr]
\end{align*}
\begin{lemma}\label{lem: conv_beta}
    The function $f(\beta):= \|\beta \bg + \blambda\|_2$ for $\beta \in \bbR$ and any $\bg, \blambda \in \bbR^{d}$ is convex in $\beta \in [0, \infty)$.
\end{lemma}
\begin{proof}
    We proceed by verifying the definition of convexity. Let $0\le t\le 1$, then for any $\beta_1, \beta_2 \ge 0$ we have
    \begin{align*}
        t f(\beta_1) + (1-t) f(\beta_2) &= \|t \beta_1 \bg + t \blambda \|_2 +   \|(1-t) \beta_1 \bg + (1-t) \blambda \|_2 \\
        &\ge \|((t \beta_1 + (1-t) \beta_2) \bg + \blambda \|_2
    \end{align*}
    Where the inequality follows by the convexity of $\ell_2$ on vectors $\beta_1 \bg +  \blambda$ and $\beta_2 \bg +  \blambda$. Hence
    \begin{align*}
        t f(\beta_1) + (1-t) f(\beta_2) \ge f(t\beta_1 + (1-t) \beta_2)
    \end{align*}
    And the result follows.
\end{proof}
We now consider $N=1$.
\subsection{End of the proof of Theorem \ref{thm: N=1}}
The system in \eqref{eq: thm_arbN_eqns} turns into
\begin{align*}
    \frac{\theta}{2}  &= 1 -  2\theta \tr \biggl[\Bigl(\frac{2}{\zeta}\bSigma+\theta \theta_1 \|\bh_1\|_2^2 \bSigma_1 \Bigr)^{-1} \bSigma \biggr] \\
    \frac{2}{\zeta} &= \theta \|\bh\|_2^2  \theta_1 \\
    \theta_1 &= 2 - 2 \theta \theta_1 \tr \biggl[\Bigl(\frac{2}{\zeta}\bSigma+\theta\theta_1 \|\bh_1\|_2^2 \bSigma_1 \Bigr)^{-1} \bSigma_1\biggr]
\end{align*}
Which implies
\begin{align*}
    \theta_1 &= 2 - \frac{2}{\|\bh\|_2^2} \tr \biggl[ \Bigl( \bSigma + \bSigma_1\Bigr)^{-1} \bSigma_1\biggr] = \frac{2n - 2\tr \biggl[\Bigl( \bSigma + \bSigma_1\Bigr)^{-1} \bSigma_1 \biggr] }{n} \\
    \theta \zeta &= \frac{2}{\theta_1 \|\bh\|_2^2} \\
    \frac{\theta}{2}  &= 1 -  \zeta \theta \tr \biggl[\Bigl( \bSigma + \bSigma_1\Bigr)^{-1} \bSigma\biggr]
\end{align*}
Hence, by plugging in for $\theta_1$, $\theta$ would be
\begin{align*}
    \theta &= 2 - \frac{2\tr \biggl[ \Bigl( \bSigma + \bSigma_1\Bigr)^{-1} \bSigma \biggr]}{\|\bh\|_2^2- \tr \biggl[ \Bigl( \bSigma + \bSigma_1\Bigr)^{-1} \bSigma_1 \biggr]}  = \frac{2n - 2d}{n - \tr \biggl[\Bigl( \bSigma + \bSigma_1\Bigr)^{-1} \bSigma_1 \biggr]} \\
    \zeta &= \frac{1}{2(n-d)}
\end{align*}
Furthermore, we have for the other variables
\begin{align*}
    \tau^2 &= \frac{\theta_1 \tau^2_1}{2} + \frac{\theta_1}{2} \zeta^2 \tau_{\lambda}^2 \|\bh\|_2^2 - \frac{\theta_1^2 \tau_1^2}{4} - 2 \partial_1  F\Bigl(\frac{\beta}{\tau}, \frac{\beta_t}{\tau_t}, \frac{\eta}{\tau_{\lambda}}, \beta, \beta_t\Bigr) \\
    \tau_1^2 &= \zeta^2 \tau_{\lambda}^2 \|\bh\|_2^2 - \frac{4}{\theta} \partial_2 F\Bigl(\theta, \theta_t, \zeta, \beta, \beta_1\Bigr) \\
    \tau^2_{\lambda} &=  \frac{1}{\zeta^2} \be_1^T \bSigma \be_1 + 2\partial_3 F\Bigl(\theta, \theta_1, \zeta, \beta, \beta_t\Bigr) 
\end{align*}
Plugging in
\begin{align*}
    \tau^2 &= \frac{\theta_1 \tau^2_1}{2} + \frac{\theta_1}{2} \zeta^2 \tau_{\lambda}^2 \|\bh\|_2^2 - \frac{\theta_1^2 \tau_1^2}{4} - \frac{\theta \zeta \theta^2_1 \tau_1^2  }{2} \tr\biggl[\Bigl(\bSigma + \bSigma_1 \Bigr)^{-1}\bSigma_1 \biggr] \\
    & + \frac{\zeta^2 \theta_1 \|\bh_1\|_2^2 }{2} \tr \biggl[\Bigl(\bSigma + \bSigma_1 \Bigr)^{-1} \bSigma_1 \Bigl(\bSigma + \bSigma_1 \Bigr)^{-1} \Bigl(\theta^2 \tau^2 \bSigma + \frac{\theta^2  \theta^2_1 \tau_1^2}{4}   \bSigma_1  + \frac{1}{\zeta^2}\bSigma \be_1 \be_1^T \bSigma\Bigr)\biggr]\\
    \tau_1^2 &= \zeta^2 \tau_{\lambda}^2 \|\bh\|_2^2 + \zeta^2 \|\bh\|_2^2\tr\biggl[\Bigl(\bSigma + \bSigma_1 \Bigr)^{-1} \bSigma_1 \Bigl(\bSigma + \bSigma_1 \Bigr)^{-1} \Bigl(\theta^2 \tau^2 \bSigma + \frac{\theta^2}{4}  \theta^2_1 \tau_1^2  \bSigma_1  + \frac{1}{\zeta^2}\bSigma \be_1 \be_1^T \bSigma\Bigr)\biggr] \\
    \tau^2_{\lambda} &=  \frac{1}{\zeta^2} \be_1^T \bSigma \be_1 +\tr \biggl[\Bigl(\bSigma + \bSigma_1 \Bigr)^{-1} \bSigma \Bigl(\bSigma + \bSigma_1 \Bigr)^{-1} \Bigl(\theta^2 \tau^2 \bSigma + \frac{\theta^2}{4}  \theta^2_1 \tau_1^2  \bSigma_1  + \frac{1}{\zeta^2}\bSigma \be_1 \be_1^T \bSigma\Bigr)\biggr] \\ 
     &-\frac{2}{\zeta^2} \tr\biggl[\Bigl(\bSigma + \bSigma_1\Bigr)^{-1}  \bSigma \be_1 \be_1^T \bSigma\biggr]
\end{align*}
Eliminating $\tau_{\lambda}^2$ yields
\begin{align*}
    \tau^2 &= \frac{\theta_1 \tau^2_1}{2} - \frac{\theta_1^2 \tau_1^2}{4} + \frac{\zeta^2 \theta_1 \|\bh_1\|_2^2 \theta^2 }{2} \tr \biggl[ \Bigl(\bSigma + \bSigma_1 \Bigr)^{-1} \bSigma \biggr] \tau^2  +  \frac{\zeta^2 \theta_1 \|\bh_1\|_2^2 }{2} \frac{\theta^2  \theta^2_1}{4} \tr \biggl[\Bigl(\bSigma + \bSigma_1 \Bigr)^{-1} \bSigma_1 \biggr]\tau_1^2 \\
    &+\frac{\theta_1 \|\bh_1\|_2^2}{2} \be_1^T \bSigma \be_1- \frac{\theta \zeta \theta^2_1 \tau_1^2  }{2} \tr\biggl[\Bigl(\bSigma + \bSigma_1 \Bigr)^{-1}\bSigma_1\biggr] - \frac{\theta_1 \|\bh\|_2^2}{2} \tr \biggl[\Bigl(\bSigma + \bSigma_1\Bigr)^{-1} \bSigma \be_1 \be_1^T \bSigma\biggr]\\
    \tau_1^2 &= \zeta^2 \|\bh\|_2^2 \theta^2 \tr\biggl[\Bigl(\bSigma + \bSigma_1 \Bigr)^{-1} \bSigma\biggr] \tau^2 + \frac{\theta^2\theta^2_1 \zeta^2 \|\bh\|_2^2}{4}   \tr\biggl[\Bigl(\bSigma + \bSigma_1 \Bigr)^{-1} \bSigma_1 \biggr] \tau_1^2 + \|\bh\|_2^2 \be_1^T \bSigma \be_1 \\
    &- \|\bh\|_2^2 \tr\biggl[\Bigl(\bSigma + \bSigma_1\Bigr)^{-1}  \bSigma \be_1 \be_1^T \bSigma\biggr]\\
    \tau^2_{\lambda} &=  \frac{1}{\zeta^2} \be_1^T \bSigma \be_1 +\tr \biggl[\Bigl(\bSigma + \bSigma_1 \Bigr)^{-1} \bSigma \Bigl(\bSigma + \bSigma_1 \Bigr)^{-1} \Bigl(\theta^2 \tau^2 \bSigma + \frac{\theta^2}{4}  \theta^2_1 \tau_1^2  \bSigma_1  + \frac{1}{\zeta^2}\bSigma \be_1 \be_1^T \bSigma\Bigr)\biggr] \\ 
     &-\frac{2}{\zeta^2} \tr\biggl[\Bigl(\bSigma + \bSigma_1\Bigr)^{-1}  \bSigma \be_1 \be_1^T \bSigma\biggr]
\end{align*}
By simplifying the expressions, we obtain
\begin{align*}
    \frac{\theta}{2}\tau^2 &= \frac{\theta_1 \|\bh_1\|_2^2}{2} \Bigl(\be_1^T \bSigma \be_1 - \tr \biggl[\Bigl(\bSigma + \bSigma_1\Bigr)^{-1} \bSigma \be_1 \be_1^T \bSigma \biggr]\Bigr)\\
      \frac{\theta_1}{2}\tau_1^2 &= \frac{2-\theta}{\theta_1}\tau^2 + \|\bh\|_2^2 \be_1^T \bSigma \be_1 - \|\bh\|_2^2 \tr\biggl[\Bigl(\bSigma + \bSigma_1\Bigr)^{-1}  \bSigma \be_1 \be_1^T \bSigma\biggr]
\end{align*}
Hence $\tau_1^2$ can be written as follows
\begin{align*}
    \tau_1^2 = \frac{4 \|\bh\|_2^2}{\theta \theta_1} \be_1^T \bSigma \be_1 - \frac{4 \|\bh\|_2^2}{\theta\theta_1} \tr\biggl[\Bigl(\bSigma + \bSigma_1\Bigr)^{-1}  \bSigma \be_1 \be_1^T \bSigma\biggr]
\end{align*}
Furthermore, we have for $\tau_{\lambda}^2$
\begin{align*}
    \zeta^2 \tau_{\lambda}^2 &=  \be_1^T \bSigma \be_1 +\tr \biggl[\Bigl(\bSigma + \bSigma_1 \Bigr)^{-1} \bSigma \Bigl(\bSigma + \bSigma_1 \Bigr)^{-1} \bSigma \be_1 \be_1^T \bSigma \biggr]- 2  \tr\biggl[\Bigl(\bSigma + \bSigma_1\Bigr)^{-1}  \bSigma \be_1 \be_1^T \bSigma\biggr]\\
    &+ 2\zeta  \Bigl(\be_1^T \bSigma \be_1 - \tr \biggl[\Bigl(\bSigma + \bSigma_1\Bigr)^{-1} \bSigma \be_1 \be_1^T \bSigma\biggr] \Bigr) \tr \biggl[\Bigl(\bSigma + \bSigma_1 \Bigr)^{-1} \bSigma\biggr] \\
    &= \frac{2}{\theta} \be_1^T \bSigma \be_1 - \frac{\theta + 2}{\theta}  \tr\biggl[\Bigl(\bSigma + \bSigma_1\Bigr)^{-1}  \bSigma \be_1 \be_1^T \bSigma\biggr]+\tr\biggl[ \Bigl(\bSigma + \bSigma_1 \Bigr)^{-1} \bSigma \Bigl(\bSigma + \bSigma_1 \Bigr)^{-1} \bSigma \be_1 \be_1^T \bSigma\biggr]
\end{align*}
Summarizing for $\eta^2 = \zeta^2 \tau_{\lambda}^2 $
\begin{align*}
     \tr \biggl[ (\bW^T - \bI)^T \bSigma &(\bW^T - \bI) \biggr]\rarrowp \frac{2}{\theta} \tr \Bigl[ \bSigma \Bigr]  - \frac{\theta + 2}{\theta}  \tr\biggl[\Bigl(\bSigma + \bSigma_1\Bigr)^{-1}  \bSigma^2 \biggr] \\
     & +\tr\biggl[\Bigl(\bSigma + \bSigma_1 \Bigr)^{-1} \bSigma \Bigl(\bSigma + \bSigma_1 \Bigr)^{-1} \bSigma^2 \biggr]\\
     &= \tr \biggl[\Bigl(\Bigl(\bSigma + \bSigma_1\Bigr)^{-1}  \bSigma - \bI \Bigr) \Bigl(\Bigl(\bSigma + \bSigma_1\Bigr)^{-1}  \bSigma - \frac{n}{n - d + \tr \Bigl(\bSigma + \bSigma_1\Bigr)^{-1}  \bSigma} \bI \Bigr) \bSigma\biggr]
\end{align*}
Denoting $\bR_{W}:= \Bigl(\bSigma + \bSigma_1\Bigr)^{-1}  \bSigma$ as the Wiener filter, we obtain
\begin{align*}
    \tr\biggl[ (\bW^T - \bI)^T \bSigma (\bW^T - \bI)\biggr] \rarrowp \tr \biggl[ \Bigl(\bR_{W} - \bI \Bigr) \Bigl(\bR_{W}  - \frac{n}{n - d + \tr \bR_{W}} \bI \Bigr) \bSigma \biggr]
\end{align*}
If $\bSigma_1 = \bSigma_z$ and $\bSigma_1$ and $\bSigma$ commute, we have for the generalization error
\begin{align*}
    \calE(\bW_{\text{lsq}}) \rarrowp \frac{n}{n-d} \Bigl( \tr \bSigma - \tr \bR_{W} \bSigma \Bigr) = \frac{n}{n-d} \cal{E}_{\text{Wiener}}
\end{align*}
More generally, we may write
\begin{align*}
     \calE(\bW_{\text{lsq}}) - \calE_{\text{Wiener}} \rarrowp  \tr \biggl[\bSigma \Bigl(\bSigma + \bSigma_1\Bigr)^{-1} \bSigma_1\biggr] \biggl[\frac{n}{n-d} \bI - \Bigl(  \bSigma_1 \Bigl(\bSigma + \bSigma_1\Bigr)^{-1} + \Bigl(\bSigma + \bSigma_1\Bigr)^{-1} \bSigma \Bigr) \biggr]
\end{align*}
The error per term is
\begin{align*}
    \zeta^2 \tau_{\lambda}^2 &+ \zeta^2 \tr\biggl[  \Bigl( \bSigma_1 + \bSigma \Bigr)^{-1} \bSigma_{\bz}  \Bigl(  \bSigma_1 + \bSigma \Bigr)^{-1} \Bigl( \theta^2 \tau^2 \bSigma + \frac{1}{\zeta^2}  \bSigma \be_1 \be_1^T \bSigma + \frac{\theta^2 \theta^2_1 \tau^2_1 }{4} \bSigma_1 \Bigr) \biggr] \\ 
    &= \zeta^2 \tau_{\lambda}^2 +  \tr \biggl[ \Bigl( \bSigma_1 + \bSigma \Bigr)^{-1} \bSigma_{\bz}  \Bigl(  \bSigma_1 + \bSigma \Bigr)^{-1}  \bSigma \be_1 \be_1^T \bSigma \biggr] \\
    &+  2\zeta  \Bigl(\be_1^T \bSigma \be_1 - \tr\biggl[ \Bigl(\bSigma + \bSigma_1\Bigr)^{-1} \bSigma \be_1 \be_1^T \bSigma  \biggr]\Bigr) \tr \biggl[\Bigl(\bSigma + \bSigma_1 \Bigr)^{-1} \bSigma_{\bz}\biggr] \\ 
    &= \be_1^T \bSigma \be_1 +\tr \biggl[ \Bigl(\bSigma + \bSigma_1 \Bigr)^{-1} \Bigl(\bSigma + \bSigma_{\bz}\Bigr) \Bigl(\bSigma + \bSigma_1 \Bigr)^{-1} \bSigma \be_1 \be_1^T \bSigma\biggr]- 2  \tr\biggl[\Bigl(\bSigma + \bSigma_1\Bigr)^{-1}  \bSigma \be_1 \be_1^T \bSigma\biggr]\\
    &+ 2\zeta  \Bigl(\be_1^T \bSigma \be_1 - \tr \biggl[\Bigl(\bSigma + \bSigma_1\Bigr)^{-1} \bSigma \be_1 \be_1^T \bSigma\biggr]  \Bigr) \tr \biggl[\Bigl(\bSigma + \bSigma_1 \Bigr)^{-1} \Bigl(\bSigma + \bSigma_{\bz}\Bigr) \biggr]
\end{align*}
Thus the total error is
\begin{align*}
    \calE(\bW_{\text{lsq}}) \rarrowp \Bigl(1 + 2\zeta \tr\biggl[ \Bigl(\bSigma + \bSigma_1 \Bigr)^{-1} \Bigl(\bSigma + \bSigma_{\bz}\Bigr) \biggr] \Bigr) \Bigl( \tr\Bigl[ \bSigma \Bigr]- \tr \biggl[ \Bigl(\bSigma + \bSigma_1\Bigr)^{-1} \bSigma^2 \biggr]\Bigr) \\
    + \tr \biggl[ \Bigl(\bSigma + \bSigma_1 \Bigr)^{-1} \Bigl(\bSigma_{\bz} - \bSigma_1\Bigr) \Bigl(\bSigma + \bSigma_1 \Bigr)^{-1} \bSigma^2 \biggr]
\end{align*}

\subsection{End of the Proof of Theorem \ref{thm: arb_N}. }
Recall that we had from \eqref{eq: thm_arbN_eqns}
\begin{align*}
     \frac{\theta}{2}  &= 1 -  2 \theta \tr\biggl[ \Bigl(\frac{2}{\zeta}\bSigma+\theta\sum_{t=1}^N \theta_t n_t \bSigma_t \Bigr)^{-1} \bSigma \biggr] \\
    \frac{2}{\zeta} &= \theta  \sum_{t=1}^N n_t \theta_t \\
    \theta_t &= 2 - 2 \theta \theta_t \tr \biggl[ \Bigl(\frac{2}{\zeta}\bSigma+\theta\sum_{t=1}^N \theta_t n_t \bSigma_t \Bigr)^{-1} \bSigma_t \biggr]
\end{align*}
This implies
\begin{align*}
     \theta_t &= 2 - 2 \theta_t \tr \biggl[\Bigl(\sum_{t=1}^N \theta_t n_t (\bSigma_t + \bSigma) \Bigr)^{-1} \bSigma_t\biggr] \\
    \zeta &= \frac{1}{2(n-d)} \\
    \frac{\theta}{2}  &= 1 -  2 \theta \tr\biggl[ \Bigl(\frac{2}{\zeta}\bSigma+\theta\sum_{t=1}^N \theta_t n_t \bSigma_t \Bigr)^{-1} \bSigma \biggr]
\end{align*}
Hence
\begin{align}\label{eq: tau}
    \frac{\theta}{2} \tau^2 = \frac{1}{\zeta \theta} \be_1^T \bSigma \be_1  - \frac{2}{\zeta^2 \theta } \tr \biggl[\Bigl(\frac{2}{\zeta}\bSigma+\theta\sum_{t=1}^N \theta_t n_t \bSigma_t \Bigr)^{-1}\bSigma \be_1 \be_1^T \bSigma \biggr] \nonumber \\
    =  \frac{1}{\zeta \theta} \be_1^T \bSigma \be_1 - \frac{2}{\zeta^2 \theta^2} \tr \biggl[ \Bigl(\sum_{t=1}^N \theta_t n_t (\bSigma_t + \bSigma) \Bigr)^{-1}\bSigma \be_1 \be_1^T \bSigma \biggr]
\end{align}
Furthermore, $\tau_i^2$ can be found through the following system of linear equations:
\begin{align}\label{eq: tau_i}
    &-  4 \tau^2 \|\bh^{(i)}\|_2^2  \tr \biggl[\Bigl(\sum_{t=1}^N \theta_t n_t (\bSigma_i + \bSigma) \Bigr)^{-1} \Bigl(\bSigma + \bSigma_t \Bigr)  \Bigl(\sum_{t=1}^N \theta_t n_t (\bSigma_t + \bSigma)\Bigr)^{-1} \bSigma \biggr] \nonumber \\
   &+\Bigl(1- \|\bh^{(i)}\|_2^2 \theta^2_i \tr \biggl[\Bigl(\sum_{t=1}^N \theta_t n_t (\bSigma_t + \bSigma) \Bigr)^{-1} \Bigl(\bSigma + \bSigma_i \Bigr)  \Bigl(\sum_{t=1}^N \theta_t n_t (\bSigma_t + \bSigma)\Bigr)^{-1}  \bSigma_i \biggr]  \Bigr) \tau_i^2 \nonumber \\
   &- \sum_{t\neq i}^N n_t \theta^2_t \tr \biggl[ \Bigl(\sum_{t=1}^N \theta_t n_t (\bSigma_t + \bSigma) \Bigr)^{-1} \Bigl(\bSigma + \bSigma_i \Bigr)  \Bigl(\sum_{t=1}^N \theta_t n_t (\bSigma_t + \bSigma)\Bigr)^{-1}  \bSigma_t \biggr] \tau_t^2  \nonumber \\ 
   &= 4 \frac{\|\bh^{(i)}\|_2^2}{\theta^2\zeta^2}  \tr \biggl[ \Bigl(\sum_{t=1}^N \theta_t n_t (\bSigma_i + \bSigma) \Bigr)^{-1} \Bigl(\bSigma + \bSigma_t \Bigr)  \Bigl(\sum_{t=1}^N \theta_t n_t (\bSigma_t + \bSigma)\Bigr)^{-1} \bSigma \be_1 \be_1^T \bSigma  \biggr]\nonumber  \\ 
   &- \frac{4  \|\bh^{(i)}\|_2^2 }{\zeta \theta} \tr \biggl[\Bigl(\sum_{t=1}^N \theta_t n_t (\bSigma_t + \bSigma) \Bigr)^{-1}  \bSigma \be_1 \be_1^T \bSigma \biggr]
\end{align}
Furthermore, we have for $\tau_{\lambda}^2$
\begin{align*}
    \tau_{\lambda}^2 &= \frac{1}{\zeta^2} \be_1^T \bSigma \be_1 + \frac{4}{\zeta^2 \theta^2} \tr \biggl[\Bigl(\sum_{t=1}^N \theta_t n_t (\bSigma_t + \bSigma) \Bigr)^{-1} \bSigma \Bigl(\sum_{t=1}^N \theta_t n_t (\bSigma_t + \bSigma) \Bigr)^{-1} \nonumber\\
      &\cdot \Bigl(\theta^2 \tau^2 \bSigma + \frac{\theta^2}{4} \sum_{t=1}^N \theta^2_t \tau_t^2 \bSigma_t  + \frac{1}{\zeta^2}\bSigma \be_1 \be_1^T \bSigma\Bigr) \biggr] \\
      &- \frac{4}{\zeta^3 \theta} \tr\biggl[\Bigl(\sum_{t=1}^N \theta_t n_t (\bSigma_t + \bSigma) \Bigr)^{-1}  \bSigma \be_1 \be_1^T \bSigma\biggr]
\end{align*}
Arranging the terms, we obtain
\begin{align}\label{eq: tau_lam}
    &- \frac{4}{\zeta^2} \tr \biggl[ \Bigl(\sum_{t=1}^N \theta_t n_t (\bSigma_t + \bSigma) \Bigr)^{-1} \bSigma \Bigl(\sum_{t=1}^N \theta_t n_t (\bSigma_t + \bSigma) \Bigr)^{-1} \bSigma \biggr] \tau \nonumber\\
      &- \frac{1}{\zeta^2} \tr \biggl[\Bigl(\sum_{t=1}^N \theta_t n_t (\bSigma_t + \bSigma) \Bigr)^{-1} \bSigma \Bigl(\sum_{t=1}^N \theta_t n_t (\bSigma_t + \bSigma) \Bigr)^{-1} \sum_{t=1}^N \theta_t^2 \bSigma_t \biggr] \tau_t^2 + \tau_{\lambda}^2 \nonumber \\ 
      &= \frac{1}{\zeta^2} \be_1^T \bSigma \be_1 + \frac{4}{\zeta^4 \theta^2} \tr \biggl[\Bigl(\sum_{t=1}^N \theta_t n_t (\bSigma_t + \bSigma) \Bigr)^{-1} \bSigma \Bigl(\sum_{t=1}^N \theta_t n_t (\bSigma_t + \bSigma) \Bigr)^{-1}  \bSigma \be_1 \be_1^T \bSigma \biggr] \nonumber \\
      &-  \frac{4}{\zeta^3 \theta} \tr\biggl[\Bigl(\sum_{t=1}^N \theta_t n_t (\bSigma_t + \bSigma) \Bigr)^{-1}  \bSigma \be_1 \be_1^T \bSigma\biggr]
\end{align}
So far all the computation presented was for the case $j=1$ as the scalar parameters depend on $\be_1$. However, we note the coefficient matrix for $\tau_t^2$, $\tau$, and $\tau_{\lambda}^2$ is independent of $\be_1$ as $\theta$ and $\theta_t$ do not depend on $\be_1$. This implies that for every $j=1,\cdots, d$, can be expressed as the solution to the following linear system of equations
\begin{align*}
    \bA \begin{pmatrix}
        \tau^{(j)2} \\
        \tau^{(j)2}_1 \\
        \vdots \\
        \tau^{(j)2}_N \\
        \tau^{(j)2}_{\lambda}
    \end{pmatrix} = \bb_j
\end{align*}
Where $\bA$ is constructed according to the equations \eqref{eq: tau}, \eqref{eq: tau_i}, and \eqref{eq: tau_lam} as follows:
\begin{align}\label{eq: A_def}
    (\bA)_{ij} = \begin{cases}
        \frac{\theta}{2}  & i = j = 1 \\
        0 & i = 1, \quad j = 2,\cdots N + 2 \\
        -  4 n_i  \tr \biggl[ \Bigl(\sum_{t=1}^N \theta_t n_t (\bSigma_i + \bSigma) \Bigr)^{-1} \Bigl(\bSigma + \bSigma_t \Bigr) & \\ \cdot \Bigl(\sum_{t=1}^N \theta_t n_t (\bSigma_t + \bSigma)\Bigr)^{-1}\biggr]  \bSigma & i = 2,\cdots N+1, \quad j = 1 \\
         -n_j \theta^2_j \tr \biggl[\Bigl(\sum_{t=1}^N \theta_t n_t (\bSigma_t + \bSigma) \Bigr)^{-1} \Bigl(\bSigma + \bSigma_i \Bigr)   & \\
         \cdot \Bigl(\sum_{t=1}^N \theta_t n_t (\bSigma_t + \bSigma)\Bigr)^{-1}  \bSigma_j \biggr] & i = 2, \cdots N+1, \quad j = 2, \cdots N+1, \quad j \neq i \\ 
         1- n_i \theta^2_i \tr \biggl[\Bigl(\sum_{t=1}^N \theta_t n_t (\bSigma_t + \bSigma) \Bigr)^{-1} \Bigl(\bSigma + \bSigma_i \Bigr)  & \\
         \cdot \Bigl(\sum_{t=1}^N \theta_t n_t (\bSigma_t + \bSigma)\Bigr)^{-1}  \bSigma_i \biggr] & i = 2, \cdots N+1, \quad j = i \\
         0 & i = 2,\cdots N+1, \quad j = N+2 \\
         - 16(n-d)^2 \tr \biggl[ \Bigl(\sum_{t=1}^N \theta_t n_t (\bSigma_t + \bSigma) \Bigr)^{-1} \bSigma & \\
         \cdot \Bigl(\sum_{t=1}^N \theta_t n_t (\bSigma_t + \bSigma) \Bigr)^{-1} \bSigma \biggr] & i = N+2, \quad j = 1 \\
         - 16(n-d)^2\theta_j^2 \tr \biggl[\Bigl(\sum_{t=1}^N \theta_t n_t(\bSigma_t + \bSigma) \Bigr)^{-1} \bSigma & \\ \cdot \Bigl(\sum_{t=1}^N \theta_t n_t (\bSigma_t + \bSigma) \Bigr)^{-1} \bSigma_j \biggr] & i = N+2, \quad j = 2, \cdots N+1 \\
         1 & i = N+2, \quad j = N+2
    \end{cases}
\end{align}

And we have for $\bb_j$
\begin{equation} \label{eq:bj}
    \bb_j = \begin{pmatrix}
        \frac{1}{\zeta \theta} \be_j^T \bSigma \be_j - \frac{2}{\zeta^2 \theta^2} \tr \biggl[ \Bigl(\sum_{t=1}^N \theta_t n_t (\bSigma_t + \bSigma) \Bigr)^{-1}\bSigma \be_j \be_j^T \bSigma \biggr]\\
        \begin{pmatrix}
            4 \frac{n_i}{\theta^2\zeta^2}  \tr \biggl[ \Bigl(\sum_{t=1}^N \theta_t n_t (\bSigma_i + \bSigma) \Bigr)^{-1} \Bigl(\bSigma + \bSigma_t \Bigr)  \Bigl(\sum_{t=1}^N \theta_t n_t (\bSigma_t + \bSigma)\Bigr)^{-1} \bSigma \be_j \be_j^T \bSigma \biggr] \\ 
            - \frac{4  n_i }{\zeta \theta} \tr\biggl[\Bigl(\sum_{t=1}^N \theta_t n_t (\bSigma_t + \bSigma) \Bigr)^{-1}  \bSigma \be_j \be_j^T \bSigma\biggr]
        \end{pmatrix}_{i=1}^N \\
        \begin{pmatrix}
            \frac{1}{\zeta^2} \be_j^T \bSigma \be_j + \frac{4}{\zeta^4 \theta^2} \tr \biggl[\Bigl(\sum_{t=1}^N \theta_t n_t (\bSigma_t + \bSigma) \Bigr)^{-1} \bSigma \Bigl(\sum_{t=1}^N \theta_t n_t (\bSigma_t + \bSigma) \Bigr)^{-1}  \bSigma \be_j \be_j^T \bSigma\biggr] \\
            -  \frac{4}{\zeta^3 \theta} \tr\biggl[\Bigl(\sum_{t=1}^N \theta_t n_t (\bSigma_t + \bSigma) \Bigr)^{-1}  \bSigma \be_j \be_j^T \bSigma\biggr]
        \end{pmatrix}
    \end{pmatrix}
\end{equation}
And we define from \eqref{eq:bj}
\begin{align} \label{eq:b_def}
    \bb := \sum_{j=1}^d \bb_j 
\end{align}

We note the total generalization error can be written as
\begin{align*}
    &\zeta^2 \sum_{j=1}^d  \tau_{\lambda}^{(j)2} + 4\theta^2 \sum_{j=1}^d \tau^{(j)2} \tr \biggl[ \Bigl(\sum_{t=1}^N \theta_t n_t (\bSigma_t + \bSigma) \Bigr)^{-1} \bSigma_{\bz}  \Bigl(\sum_{t=1}^N \theta_t n_t (\bSigma_t + \bSigma) \Bigr)^{-1} \bSigma \biggr] \\
    &+ \frac{4}{\zeta^2} \tr \biggl[ \Bigl(\sum_{t=1}^N \theta_t n_t (\bSigma_t + \bSigma) \Bigr)^{-1} \bSigma_{\bz}  \Bigl(\sum_{t=1}^N \theta_t n_t (\bSigma_t + \bSigma) \Bigr)^{-1} \bSigma^2 \biggr]\\
    &+ \theta^2  \tr \biggl[ \Bigl(\sum_{t=1}^N \theta_t n_t (\bSigma_t + \bSigma) \Bigr)^{-1} \bSigma_{\bz}  \Bigl(\sum_{t=1}^N \theta_t n_t (\bSigma_t + \bSigma) \Bigr)^{-1}  \sum_{t=1}^N \theta^2_t \bSigma_t \biggr] \sum_{j=1}^d \tau^{(j)2}_t 
\end{align*}
We observe that the generalization error is a function of $\sum_{j=1}^d \tau^{(j)2}_t$, $\sum_{j=1}^d  \tau_{\lambda}^{(j)2}$, $ \sum_{j=1}^d \tau^{(j)2}$. Hence we can write the generalization error as
\begin{align*}
    \zeta^2 (\bA^{-1} \bb)_{N+2} &+ 4\theta^2 (\bA^{-1} \bb)_{1} \tr \biggl[ \Bigl(\sum_{t=1}^N \theta_t n_t (\bSigma_t + \bSigma) \Bigr)^{-1} \bSigma_{\bz}  \Bigl(\sum_{t=1}^N \theta_t n_t (\bSigma_t + \bSigma) \Bigr)^{-1} \bSigma \biggr] \\
     &+ \frac{4}{\zeta^2} \tr \biggl[ \Bigl(\sum_{t=1}^N \theta_t n_t (\bSigma_t + \bSigma) \Bigr)^{-1} \bSigma_{\bz}  \Bigl(\sum_{t=1}^N \theta_t n_t (\bSigma_t + \bSigma) \Bigr)^{-1} \bSigma^2 \biggr]\\
     &+ \theta^2  \tr \biggl[  \Bigl(\sum_{t=1}^N \theta_t n_t (\bSigma_t + \bSigma) \Bigr)^{-1} \bSigma_{\bz}  \Bigl(\sum_{t=1}^N \theta_t n_t (\bSigma_t + \bSigma) \Bigr)^{-1}  \sum_{t=1}^N \theta^2_t \bSigma_t (\bA^{-1} \bb)_{t+1} \biggr]  
\end{align*}
Where $(\bA^{-1} \bb)_j$ denotes the $j$th entry of the vector $\bA^{-1} \bb \in \bbR^{N+2}$.
\subsection{Proof of Concentration of the Generalization Error}

In this section, we would like to show that one can use the concentration of some scalar functions of $\bw_{j}^{\ast}$ for $j = 1, \hdots, d$ to prove the concentration of the total generalization error of $\bW_{\text{lsq}}$. Namely, by the prior arguments, we have observed that $\|\bSigma_j^{1/2}(\bw^\ast_j-\be_j)\|^2_2 \rarrowp \zeta^2 \tau_{\lambda}^{(j)2}$ for each $j = 1, \hdots, d$. Moreover, one can specifiy the concentration rate, similar to \cite{dokmanic2019concentration, hassani2024curse}, where one can show that there exist $c_1, c_2= \theta(1)$ such that for $t_j > 0$ for $i\in [d]$:
\begin{align*}
    \bbP\Bigl(\Bigl|d \|\bSigma_j^{1/2}(\bw^\ast_j-\be_j)\|^2_2 - d\zeta^2 \tau_{\lambda}^{(j)2}  \Bigr| > t_j\Bigr) \le c_1 \exp(- c_2 d t_j)
\end{align*}
Furthermore, let
\begin{align*}
    \calL_j :=  \tr \biggl[ \Bigl(  \theta \sum_{t=1}^N \theta_t n_t \bSigma_t + \frac{2}{\zeta} \bSigma \Bigr)^{-1} \bSigma_{\bz}  &\Bigl(  \theta \sum_{t=1}^N \theta_t n_t \bSigma_t + \frac{2}{\zeta} \bSigma \Bigr)^{-1} \\
    &\cdot \Bigl( 4 \theta^2 \tau^2 \bSigma + \frac{4}{\zeta^2} \bSigma \be_j \be_j^T \bSigma + \theta^2 \sum_{t=1}^N \theta^2_t \tau^{(j)2}_t \bSigma_t \Bigr)\biggr]
\end{align*}
We have similarly for $c_3, c_4 = \theta(1)$
\begin{align*}
     \bbP\Bigl(\Bigl| d\bw^{\ast T}_j\bSigma_{\bz}\bw^\ast_j - d\calL_j  \Bigr| > t_j\Bigr) \le c_3 \exp(- c_4 d t_j)
\end{align*}
Note that we had
\begin{align*}
    \calE(\bW_{\text{lsq}}) = \sum_{j=1}^d \bw^{\ast T} \bSigma_{\bz} \bw_j^\ast + (\bw_j^\ast - \be_j)^T \bSigma_j (\bw_j^\ast - \be_j)
\end{align*}
The concentration of the solutions follows by a union bound argument, from:
\begin{align*}
    \bbP\Bigl(\Bigl| \calE(\bW_{\text{lsq}}) - \sum_{j=1}^d \calL_j + \zeta^2 \tau_{\lambda}^{(j)2} \Bigr| > t \Bigr) \le \sum_{j=1}^d \bbP\Bigl(\Bigl| \bw^{\ast T}_j\bSigma_{\bz}\bw^\ast_j - \calL_j +   \|\bSigma_j^{1/2}(\bw^\ast_j-\be_j)\|^2_2 - \zeta^2 \tau_{\lambda}^{(j)2} \Bigr| > \frac{t}{d} \Bigr) \\
    \le 2d \max\{c_1, c_3\} \exp(-\min\{c_2,c_4\}dt)
\end{align*}
Thus the statement follows as $t = \omega(1)$.
\section{Proof of Corollary \ref{cor: scal}}
We consider $\bSigma = c \bI$ and $\bSigma_{\bz} = c_{\bz} \bI$ and write
\begin{align*}
    \min_{\bSigma_1 \succeq \bzero}& \Bigl(1 + \frac{c+c_{\bz}}{n-d} \tr\Bigl( c \bI + \bSigma_1 \Bigr)^{-1}\Bigr) \Bigl(cd - c^2  \tr\Bigl( c \bI + \bSigma_1 \Bigr)^{-1} \Bigr) \\
    &+ c^2 \tr\biggl[\Bigl( c \bI + \bSigma_1 \Bigr)^{-1}\Bigl( c_{\bz} \bI - \bSigma_1 \Bigr) \Bigl( c \bI + \bSigma_1 \Bigr)^{-1}\biggr]
\end{align*}
Note that, since the KKT condition for the optimal $\bSigma_1$ is symmetric w.r.t. the eigenvalues of $\bSigma_1$ in this case, we can take  $\bSigma_1 = \sigma \bI$ , which leads to
\begin{align*}
    c \cdot &\min_{\sigma \ge 0} \Bigl(1 + \frac{d}{n-d} \frac{c+c_{\bz}}{c+\sigma}\Bigr) \Bigl(d - d\frac{c}{c+\sigma} \Bigr) + c d \frac{c_{\bz}- \sigma}{(c+\sigma)^2} \\ 
    &= c d \cdot \min_{\sigma \ge 0}  1 + \frac{\frac{d (c+c_{\bz})}{n-d} - c}{c+\sigma} + \frac{-c\frac{d (c+c_{\bz})}{n-d} + c(c_{\bz}- \sigma)}{(c+\sigma)^2} \\ 
    &= cd + cd \cdot \min_{\sigma \ge 0}  \frac{-c\frac{d (c+c_{\bz})}{n-d} + c(c_{\bz}- \sigma) + \Bigl(\frac{d (c+c_{\bz})}{n-d} - c\Bigr) (c+\sigma)}{(c+\sigma)^2} \\
    &=  cd +  cd \cdot \min_{\sigma \ge 0}  \frac{\Bigl(\frac{d (c+c_{\bz})}{n-d} - 2c\Bigr) \sigma + cc_{\bz} - c^2}{(c+\sigma)^2} 
\end{align*}
Then we observe for the optimal $\sigma$
\begin{align*}
    \sigma^\ast = \begin{cases}
        0 & (3-2\kappa) c + c_{\bz} < 0 ,\quad  c + (3 - 2\kappa) c_{\bz} > 0 \\
        c \frac{ c + (3 - 2\kappa) c_{\bz}}{(3-2\kappa) c + c_{\bz}} & (3-2\kappa) c + c_{\bz} < 0 , \quad c + (3 - 2\kappa) c_{\bz} < 0 \\
        \infty & (3-2\kappa) c + c_{\bz} > 0 ,\quad  c + (3 - 2\kappa) c_{\bz} < 0 \\
        \infty & (3-2\kappa) c + c_{\bz} > 0 ,\quad  c + (3 - 2\kappa) c_{\bz} > 0, \quad c_{\bz} > c \\
        0 & (3-2\kappa) c + c_{\bz} > 0 ,\quad  c + (3 - 2\kappa) c_{\bz} > 0, \quad c_{\bz} < c
    \end{cases}
\end{align*}
In each case, the best generalization error turns out to be
\begin{align*}
    g.e \rarrowp \begin{cases}
        d c_{\bz} & (3-2\kappa) c + c_{\bz} < 0 ,\quad  c + (3 - 2\kappa) c_{\bz} > 0 \\
        d \frac{-(c+ c_{\bz})^2 +  4(\kappa-1)^2 c c_{\bz}}{4 (\kappa-2) (\kappa-1) ( c + c_{\bz}) } & (3-2\kappa) c + c_{\bz} < 0 , \quad c + (3 - 2\kappa) c_{\bz} < 0 \\
        d c & (3-2\kappa) c + c_{\bz} > 0 ,\quad  c + (3 - 2\kappa) c_{\bz} < 0 \\
        d c & (3-2\kappa) c + c_{\bz} > 0 ,\quad  c + (3 - 2\kappa) c_{\bz} > 0, \quad c_{\bz} > c \\
        d c_{\bz} & (3-2\kappa) c + c_{\bz} > 0 ,\quad  c + (3 - 2\kappa) c_{\bz} > 0, \quad c_{\bz} < c
    \end{cases}
\end{align*}

\end{document}